\documentclass[runningheads]{llncs}

\usepackage[T1]{fontenc}
\usepackage[utf8]{inputenc}
\usepackage[numbers,sort&compress]{natbib}
\usepackage{hyperref}
\usepackage{url}
\usepackage{booktabs}
\usepackage{amsfonts}
\usepackage{amsmath}
\usepackage{amssymb}
\usepackage{nicefrac}
\usepackage{microtype}
\usepackage{xcolor}
\usepackage{multirow}
\usepackage{graphicx}
\graphicspath{{./}{../}}
\usepackage{enumitem}

\begin{document}

\title{Expected Free Energy as Belief-Dependent Utility for $\rho$-POMDPs}

\titlerunning{EFE as Belief-Dependent Utility for $\rho$-POMDPs}

\author{Patrick Cooper \and Alvaro Velasquez}

\authorrunning{P. Cooper and A. Velasquez}

\institute{Department of Computer Science, University of Colorado Boulder,\\ Boulder, CO, USA\\
\email{\{patrick.cooper,alvaro.velasquez\}@colorado.edu}}

\maketitle

\begin{abstract}
An agent acting under partial observability must decide when to gather information and which observations are worth their cost. Standard POMDPs value information only through its eventual effect on reward. The $\rho$-POMDP framework instead rewards uncertainty reduction directly, through a belief-dependent utility $\rho$, but in practice both the choice of $\rho$ and the weight placed on it are tuned by hand for every task. We show that active inference removes this tuning entirely. Minimizing Expected Free Energy (EFE) is exactly equivalent to solving a $\rho$-POMDP whose utility is expected information gain, and the exploration weight is fixed at $w{=}1$ because the variational bound expresses pragmatic and epistemic value in the same units (nats). We prove this equivalence for observe-then-commit POMDPs and extend it to factored observation POMDPs, a broader class that covers interleaved observe-act problems such as non-destructive testing and mobile sensing, where gathering information leaves the hidden state unchanged. Experiments support the theory. Across environments ranging from the classic Tiger problem to RockSample and a new Structural Inspection benchmark with over $65{,}000$ states, the untuned weight matches or outperforms reward-only planning at the same horizon, avoids the over-exploration of bonuses tuned per task, and sits near the reward-maximizing knee of the success--reward Pareto frontier. The practical payoff is an exploration objective that works out of the box. In applications such as fault detection and medical screening, where every test has a price and every missed fault has a cost, EFE supplies a belief-dependent utility that is derived rather than tuned.
\end{abstract}

\section{Introduction}

Decision-making under partial observability requires agents to balance exploiting current knowledge against gathering information to reduce uncertainty about hidden states. In standard POMDPs, information gathering has no intrinsic value. It is useful only insofar as it leads to higher expected reward. This creates a well-known difficulty: the exploration--exploitation trade-off must be resolved either by the planning horizon or by heuristic exploration bonuses.

The $\rho$-POMDP framework \citep{araya2010} addresses this by extending POMDPs with a belief-dependent utility $\rho(b)$ that allows the agent to derive value directly from properties of its belief state. This enables explicit optimization over uncertainty reduction, information gain, or other belief-state properties alongside task reward. However, the choice of $\rho$ remains largely heuristic and environment-specific.

Separately, the Active Inference (AIF) framework \citep{friston2010, parr2019} casts perception and action as approximate Bayesian inference, selecting policies that minimize Expected Free Energy (EFE). EFE naturally decomposes into a pragmatic term (goal-seeking) and an epistemic term (information-seeking). Because both terms come from a single variational objective, their relative scale is not free to choose. In the discrete-state formulation we use, the coefficient on expected information gain is fixed at $w{=}1$ when both terms are measured in nats, the natural-log unit of information (Proposition~\ref{prop:equivalence}). The weight is a consequence of the derivation, not a hyperparameter.

We propose substituting EFE as $\rho$ in $\rho$-POMDPs, yielding an agent whose epistemic foraging is a consequence of its objective rather than an engineered bonus.\footnote{Code and experiment data are available at \url{https://github.com/PatrickAllenCooper/rho_aif}.} Our contributions are:
\begin{enumerate}
    \item \textbf{Theory.} A formal bridge between $\rho$-POMDPs and active inference. We prove the equivalence for observe-then-commit POMDPs (Proposition~\ref{prop:equivalence}), characterize when the canonical weight is near-optimal (Proposition~\ref{prop:nearopt}), and extend the equivalence to factored observation POMDPs, interleaved settings where information gathering preserves the hidden state (Proposition~\ref{prop:factored}).
    \item \textbf{Evidence.} Controlled comparisons against same-horizon planning, tuned information gain, and POMCP \citep{silver2010} across six observe-then-commit environments and four instances of the standard RockSample benchmark \citep{smith2004}. A Pareto analysis shows that $w{=}1$ Pareto-dominates same-horizon planning without per-environment search.
    \item \textbf{Practical guidance.} A characterization of when EFE-as-$\rho$ helps. The advantage appears when the agent must choose among multiple observation actions, and it grows with state space size ($66.5\%$ vs.\ $2.5\%$ success on Tileworld $8{\times}8$) and with the number of observation actions ($+7.82$ reward on RockSample[7,8]).
\end{enumerate}

\section{Related work}

\paragraph{POMDPs and solvers.}
POMDPs formalize sequential decision-making under state uncertainty \citep{smallwood1973, kaelbling1998}. Exact solutions are PSPACE-complete. Point-based offline methods (PBVI \citep{pineau2003}, HSVI \citep{smith2004}, SARSOP \citep{kurniawati2008}) approximate the value function on reachable beliefs \citep{shani2013}. Online solvers plan from the current belief: POMCP \citep{silver2010} uses MCTS with UCB1 and rollout evaluation, DESPOT \citep{ye2017} searches a regularized sparse belief tree, and POMCPOW \citep{sunberg2018} extends POMCP with progressive widening for continuous spaces. All explore implicitly through stochastic simulations rather than explicitly valuing information gain. Section~\ref{sec:discussion} and Appendix~\ref{app:pomcp} demonstrate the benefit of closed-form information valuation.

\paragraph{$\rho$-POMDPs.}
\citet{araya2010} introduced $\rho$-POMDPs, augmenting the reward with a belief-dependent utility $\rho : \Delta(S) \to \mathbb{R}$, so the objective becomes $\max_\pi \mathbb{E}_\pi[\sum_t \gamma^t (R(s_t,a_t) + \rho(b_t))]$. When $\rho$ is convex, the value function remains piecewise linear and convex (PWLC), preserving compatibility with standard solvers. \citet{fehr2018} extended this to Lipschitz-continuous non-convex $\rho$, and \citet{benchetrit2025} developed $\rho$-POMCPOW for continuous-space $\rho$-POMDPs. Common choices for $\rho$ (entropy reduction, KL divergence from a target belief, and information gain) each encode a different notion of epistemic value, but the choice remains heuristic and environment-specific. Our contribution is to derive a principled $\rho$ from the variational bound of active inference, fixing the exploration weight without per-environment search.

\paragraph{Value of information and experimental design.}
The idea that information has quantifiable decision-theoretic value predates both POMDPs and active inference. \citet{howard1966} formalized the value of information in decision analysis, and \citet{lindley1956} introduced expected information gain as a criterion for optimal Bayesian experimental design. In the bandit setting, information-directed sampling (IDS) \citep{russo2014} explicitly trades off instantaneous regret against information gain by minimizing the information ratio $\Gamma_t = \delta_t^2 / g_t$, where $\delta_t$ is the expected regret and $g_t$ is the information gain. The key structural difference from EFE is that IDS minimizes the ratio at each step (a relative weighting that adapts to the current belief), while EFE fixes the weight at $w{=}1$ (an absolute weighting derived from the variational bound). An IDS baseline adapted to the observe-then-commit structure would be informative but is beyond the scope of this work. Our observe-then-commit structure parallels sequential Bayesian experimental design, where the agent selects experiments (observations) to maximize information about an unknown state before making a terminal decision. The $\rho$-POMDP framework with $\rho = I(b)$ operationalizes this connection. Our contribution is to show that EFE derives a canonical weight for the information gain term from first principles rather than treating it as a tunable parameter.

\paragraph{Active inference and EFE.}
The Active Inference (AIF) framework \citep{friston2010} casts perception and action as variational inference under the free-energy principle. \citet{parr2022} provide a comprehensive textbook treatment. \citet{friston2015} formalized the decomposition of policy value into extrinsic (goal-seeking) and epistemic (information-seeking) components, showing that curiosity-driven exploration arises automatically from expected free energy minimization. \citet{dacosta2020} synthesized discrete-state AIF from first principles, showing that the posterior over policies takes the form $Q(\pi) \propto \exp(-\mathcal{G}(\pi))$ where $\mathcal{G}$ is the Expected Free Energy. \citet{parr2019} showed that EFE decomposes into pragmatic value (divergence from preferred observations) and epistemic value (expected information gain), with both arising from a single variational bound, requiring no tunable exploration weight. Despite different constructions, EFE and Generalised Free Energy produce identical policy posteriors. The mathematical foundations of EFE have been critically examined: \citet{millidge2021} showed that naively extending the variational free energy into the future does not yield exploratory behavior, proposing the Free Energy of the Expected Future (FEEF) as an alternative with clearer mathematical grounding. \citet{champion2024} addressed the ``unification problem'' by formalizing how multiple EFE formulations relate to a single root definition under different assumptions about prior preferences. \citet{devries2025} recast EFE-based planning as entropy-corrected variational inference with message-passing schemes, providing an alternative derivation that connects EFE to variational message passing.

\paragraph{Sophisticated inference.}
Standard AIF evaluates policies myopically. \citet{friston2021} introduced sophisticated inference, a recursive extension implementing deep tree search over belief trajectories rather than states. Sophistication (maintaining beliefs about future beliefs) enables counterfactual reasoning about the downstream epistemic consequences of actions. \citet{dacosta2020b} proved that this recursive scheme recovers Bellman-optimal policies for any finite horizon, whereas standard AIF achieves optimality only for single-step planning. Our recursive EFE agent (Equation~\ref{eq:efe_recursive}) is derived from this framework, adapted to the $\rho$-POMDP commit-action structure. The recursive counterfactual reasoning over future beliefs is preserved. What changes is that our observe-then-commit setting does not involve state transitions between time steps, simplifying the belief-trajectory computation.

\paragraph{Scaling active inference.}
Discrete-state AIF with full policy enumeration is limited to small state--action spaces. Several lines of work address scaling. \citet{fountas2020} combined deep generative models with Monte Carlo tree search for EFE-optimal planning in continuous state spaces. \citet{tschantz2020} developed an RL-compatible objective (the free energy of the expected future) that inherits AIF's exploration--exploitation balance while scaling to standard RL benchmarks. \citet{maisto2025} combined active inference with Monte-Carlo tree search for large POMDPs, achieving state-of-the-art on RockSample \citep{smith2004}. Our MCTS-EFE variant (Section~6) follows this direction, using EFE as a leaf heuristic within MCTS to extend planning horizons beyond exact tree search.

\paragraph{Exploration, control as inference, and intrinsic motivation.}
The control-as-inference perspective \citep{todorov2007, levine2018} casts reward maximization as variational inference. Maximum-entropy RL \citep{haarnoja2018} and stochastic optimal control \citep{rawlik2012} are algorithmic instances. \citet{millidge2020} proved formal equivalence between AIF and control-as-inference. \citet{sajid2021} showed that intrinsic motivation behaviors arise under EFE. Intrinsic motivation methods, including curiosity \citep{schmidhuber1991, pathak2017}, typologies of intrinsic signals \citep{oudeyer2007}, Bayesian surprise \citep{itti2009}, count-based exploration \citep{bellemare2016}, VIME \citep{houthooft2016}, and random network distillation \citep{burda2019}, all require tunable bonus weights. Bayes-Adaptive MDPs \citep{duff2002, guez2013} and Bayesian RL more broadly \citep{ghavamzadeh2015} address model uncertainty (unknown dynamics), distinct from our focus on state uncertainty under a known model. Our $\rho$-POMDP formalism makes the connection between EFE and these lines of work precise, inheriting formal properties (convexity, Lipschitz continuity) while fixing the exploration weight from first principles.

\section{Methodology}
\label{sec:methodology}

\subsection{The $\rho$-POMDP framework}

We restrict attention to observe-then-commit $\rho$-POMDPs, in which the action set $\mathcal{A}$ partitions into observation actions $\mathcal{A}_\text{obs}$ (which update the belief at known cost but do not change the hidden state) and terminal commit actions $\mathcal{A}_\text{com}$ (which end the episode with state-dependent reward). An episode consists of a variable-length sequence of observation actions followed by a single commit action. The hidden state is fixed throughout.

A $\rho$-POMDP extends the standard POMDP with a belief-dependent utility $\rho : \Delta(S) \rightarrow \mathbb{R}$. The agent's objective becomes:
\begin{equation}
    \pi^* = \arg\max_\pi \mathbb{E}_\pi \left[ \sum_{t=0}^{H} \gamma^t \left( R(s_t, a_t) + \rho(b_t) \right) \right]
\end{equation}
When $\rho = 0$, we recover the standard POMDP. When $\rho$ encodes information gain, the agent is explicitly rewarded for reducing uncertainty.

\subsection{Expected Free Energy as $\rho$}

In the standard AIF formulation, the EFE for a policy $\pi$ at future time $\tau$ is:
\begin{equation}
    \mathcal{G}(\pi) = \underbrace{- \mathbb{E}_{Q(o_\tau|\pi)}[\ln P(o_\tau|C)]}_{\text{Pragmatic value}} - \underbrace{\mathbb{E}_{Q(o_\tau|\pi)}\left[D_{\mathrm{KL}}\left[Q(s_\tau|o_\tau, \pi) \,\|\, Q(s_\tau|\pi)\right]\right]}_{\text{Epistemic value}}
\end{equation}
where $P(o_\tau|C)$ encodes preferred outcomes and $D_{\mathrm{KL}}$ measures expected information gain. In our observe-then-commit setting, the pragmatic and epistemic terms take concrete forms. For commit actions, the pragmatic term reduces to expected reward under the current belief: $\mathcal{G}(\text{commit}_i) = -\mathbb{E}_b[R_i]$. We use the reward matrix directly rather than encoding rewards through preferred outcome distributions $P(o|C)$, avoiding the preference-calibration issue flagged as a source of hidden tuning in prior AIF work. For observation actions, the pragmatic term is the known observation cost $c_k$, and the epistemic term is the expected information gain $I_k(b) = H(b) - \mathbb{E}_o[H(b'_o \mid \text{obs}_k)]$.

Following the sophisticated inference scheme of \citet{friston2021}, the EFE agent evaluates actions via a recursive tree search over belief states:
\begin{equation}
    \mathcal{G}(\text{observe}_k) = c_k - I_k(b) + \mathbb{E}_{o}\!\left[\min_a \mathcal{G}(a \mid b'_o)\right]
    \label{eq:efe_recursive}
\end{equation}
The agent selects $\arg\min_a \mathcal{G}(a)$. Standard AIF introduces a precision parameter $\beta$ via a softmax policy $Q(\pi) \propto \exp(-\beta \mathcal{G}(\pi))$. Our deterministic $\arg\min$ corresponds to $\beta \to \infty$, which eliminates this tunable knob. Equation~\ref{eq:efe_recursive} contains no separate exploration weight, and this is not a modeling shortcut. When EFE is written in nats, expected information gain enters on the same footing as the KL terms that define the variational objective. This shared scale is why the $\rho$-POMDP reduction carries $w{=}1$ exactly (Proposition~\ref{prop:equivalence}).

\subsection{Formal equivalence with $\rho$-POMDPs}
\label{subsec:formal_rho_equiv}

We generalize the standard $\rho$-POMDP formulation to action-dependent belief utilities $\rho(b, a)$, where the augmented reward becomes $R(s,a) + \rho(b,a)$. This is equivalent to a standard $\rho(b)$ formulation on an augmented belief-action space but avoids notational overhead. The structural results of \citet{araya2010} carry over when $\rho(\cdot, a)$ satisfies the relevant conditions for each fixed $a$.

Define $V(a, b, d) \triangleq -\mathcal{G}(a, b, d)$. Then Equation~\ref{eq:efe_recursive} becomes:
\begin{align}
    V(\text{observe}_k, b, d) &= -c_k + I_k(b) + \mathbb{E}_o\!\left[\max_a V(a, b'_o, d{+}1)\right] \label{eq:bellman_obs} \\
    V(\text{commit}_i, b, d) &= \mathbb{E}_b[R_i] \label{eq:bellman_commit}
\end{align}
This is exactly the Bellman recursion for a $\rho$-POMDP (Equation 1) with action-dependent belief utility:

\begin{proposition}
\label{prop:equivalence}
Define $\rho_{\mathrm{EFE}}(b, a)$ as:
\begin{equation*}
    \rho_{\mathrm{EFE}}(b, a) = \begin{cases} I_a(b) & \text{if } a \text{ is an observation action} \\ 0 & \text{if } a \text{ is a commit action} \end{cases}
\end{equation*}
where $I_a(b) = H(b) - \mathbb{E}_{o|a}[H(b'_o)]$ is the expected information gain from observation action $a$ at belief $b$. Then for undiscounted finite horizon ($\gamma{=}1$), minimizing recursive EFE (Eq.~\ref{eq:efe_recursive}) over horizon $H$ produces the same policy as solving the $\rho$-POMDP Bellman equation $V^*(b) = \max_a \{R(b,a) + \rho_{\mathrm{EFE}}(b,a) + \mathbb{E}_o[V^*(b'_o)]\}$ over the same horizon.
\end{proposition}
\begin{proof}[Proof sketch]
The negation $V = -\mathcal{G}$ converts $\arg\min \mathcal{G}$ to $\arg\max V$. Substituting into Equation~\ref{eq:efe_recursive}: for observation actions, $V = -c_k + I_k(b) + \mathbb{E}_o[\max_{a'} V(a', b'_o)]$, which matches the $\rho$-POMDP Bellman backup with $R(b, \text{obs}_k) = -c_k$ and $\rho = I_k(b)$. For commit actions, $V = \mathbb{E}_b[R_i]$ with $\rho = 0$, matching a terminal $\rho$-POMDP action. The recursive structure is identical, so the policies agree at every belief node.
\end{proof}

Proposition~\ref{prop:equivalence} makes precise what EFE-as-$\rho$ means: the epistemic term of EFE functions as an action-dependent belief utility. Crucially, this is equivalent to Planning+IG with $w{=}1$ (in nats). EFE does not eliminate the weight. It derives a canonical weight from the variational bound, fixing $w{=}1$ without per-environment search. Whether this canonical choice is near-optimal is an empirical question we address in Section~\ref{sec:pareto}. The following result characterizes the conditions under which $w{=}1$ is near-optimal for expected reward.

\begin{proposition}
\label{prop:nearopt}
Consider a two-state observe-then-commit $\rho$-POMDP with uniform prior $b(s_0){=}b(s_1){=}\tfrac{1}{2}$, a single observation action (accuracy $p > \tfrac{1}{2}$, cost $c > 0$), and two commit actions (correct reward $R^+$, incorrect penalty $R^- < 0$ with $|R^-| > R^+$). Define the reward asymmetry ratio $\alpha = |R^-|/R^+$ and the informativeness ratio $\eta = I_{\max}/c$ where $I_{\max} = \ln 2 - H_{\text{post}}(p)$ is the maximum expected information gain in nats. At $H{=}1$, the minimum weight $w^*_{\mathrm{thresh}}$ at which observing yields higher expected reward than committing immediately is:
\begin{equation*}
    w^*_{\mathrm{thresh}} = \frac{c \;-\; \left(p - \tfrac{1}{2}\right)(R^+ - R^-)}{\,I_{\max}\,} = \frac{c \;-\; \left(p - \tfrac{1}{2}\right)(1 + \alpha)\,R^+}{\,I_{\max}\,}
\end{equation*}
For any $w > w^*_{\mathrm{thresh}}$ the agent observes before committing, yielding identical (and higher) expected reward. The threshold $w^*_{\mathrm{thresh}}$ is negative, making $w{=}1$ trivially sufficient, whenever $\alpha > c / [(p - \tfrac{1}{2})\,R^+] - 1$. In particular, $w^*_{\mathrm{thresh}} \to -\infty$ as $\alpha \to \infty$ for any fixed $p > \tfrac{1}{2}$ and $c > 0$: high reward asymmetry makes observation so valuable that any positive weight suffices.
\end{proposition}
\begin{proof}[Proof sketch]
At $H{=}1$, the agent observes once and commits. A Planning+IG agent with weight $w$ observes iff $-c + w \cdot I(b) + \max_i \mathbb{E}_{b'_o}[R_i] > \max_i \mathbb{E}_b[R_i]$, where the left side is the observe-then-commit value and the right side is the immediate commit value. For the uniform prior, the immediate commit value is $(R^+ + R^-)/2$. After one observation, the posterior concentrates: with probability $p$ the agent is correct, yielding expected commit value $p \cdot R^+ + (1-p) \cdot R^-$. The net gain from observing is $(p - \frac{1}{2})(R^+ - R^-) - c + w \cdot I_{\max}$. Setting this to zero gives $w^*_{\mathrm{thresh}} = [c - (p-\frac{1}{2})(R^+ - R^-)]/I_{\max}$. Substituting $R^- = -\alpha R^+$ yields the second form. When $\alpha \gg 1$, the marginal reward improvement $(p - \frac{1}{2})(1 + \alpha)R^+$ dominates the cost $c$, making $w^*_{\mathrm{thresh}}$ negative, so the agent should observe at any $w \geq 0$, including $w{=}1$.
\end{proof}

Table~\ref{tab:alpha_eta} validates Proposition~\ref{prop:nearopt} across our environments, computing $\alpha$, $\eta$, the threshold $w^*_{\text{thresh}}$, and comparing against the observed reward-maximizing weight from the Pareto sweep (Section~\ref{sec:pareto}).

\begin{table}[ht]
\centering
\caption{Reward asymmetry ($\alpha$), informativeness ($\eta$), observation threshold $w^*_{\text{thresh}}$ from Proposition~\ref{prop:nearopt}, and observed reward-maximizing weight from the Pareto sweep. When $w^*_{\text{thresh}} < 0$, any positive weight (including $w{=}1$) induces observation.}
\label{tab:alpha_eta}
\begin{small}
\begin{tabular}{lcccccc}
\toprule
Environment & $\alpha$ & $\eta$ & $w^*_{\text{thresh}}$ & $w{=}1$ sufficient? & $w^*_{\text{ret}}$ (observed) \\
\midrule
Testbed & $1.0$ & $7.5$ & $+0.60$ & Yes, but over-explores & $0.5$ \\
Tiger & $10.0$ & $0.6$ & $-5.04$ & Yes ($\alpha \gg 1$) & $1.0$ \\
Diagnosis & $5.0$ & $0.5$ & $-2.34$ & Yes ($\alpha \geq 5$) & $0.5$ \\
Bandit & $1.1$ & $1.2$ & $-0.12$ & Yes (marginal) & $1.0$ \\
Tileworld & $5.0$ & $0.5$ & $-2.34$ & Yes ($\alpha \geq 5$) & $0.5$ \\
\bottomrule
\end{tabular}
\end{small}
\end{table}

The proposition is stated for $H{=}1$ and two states. The multi-step case is harder to analyze because the observation threshold shifts with belief, but the qualitative prediction holds: environments where $\alpha \geq 5$ (Tiger, Diagnosis, Tileworld) have $w^*_{\text{thresh}} \ll 0$, meaning $w{=}1$ is far above the threshold and near-optimal. The low-asymmetry Testbed ($\alpha{=}1$) has $w^*_{\text{thresh}} > 0$ but below 1, so $w{=}1$ still induces observation but assigns more weight to information than is instrumentally optimal, consistent with our finding that EFE over-explores there (Appendix~\ref{app:testbed}). On the Bandit ($\alpha{=}1.1$), $w^*_{\text{thresh}}$ is slightly negative at $H{=}1$. At $H{>}1$, multi-step planning amplifies the value of each observation, bringing the effective reward-maximizing weight toward 1.

\paragraph{Extension to discounting.} With $\gamma < 1$, the recursive EFE becomes $\mathcal{G}(\text{obs}_k) = c_k - I_k(b) + \gamma\, \mathbb{E}_o[\min_a \mathcal{G}(a, b'_o)]$, giving $V(\text{obs}_k) = -c_k + I_k(b) + \gamma\, \mathbb{E}_o[\max_a V(a, b'_o)]$. The information gain term $I_k(b)$ appears undiscounted at the current step, while future values are discounted. This preserves the $\rho$-POMDP equivalence with $\rho(b,a) = I_a(b)$, but the effective ratio of epistemic to pragmatic weight increases at early steps relative to late steps. In the undiscounted case ($\gamma{=}1$), both terms are weighted equally at every depth. With $\gamma < 1$, the agent places relatively more value on immediate information gain compared to future reward, producing slightly more exploratory behavior at early steps. Our experiments (Appendix~\ref{app:discount}) confirm that performance is robust across $\gamma \in \{0.9, 0.95, 0.99, 1.0\}$.

\paragraph{Extension to factored observation POMDPs.}
\label{par:frontier}
Proposition~\ref{prop:equivalence} assumes the observe-then-commit structure. We now extend the equivalence to a broader class that includes interleaved observe-act POMDPs.

\begin{definition}[Factored observation POMDP]
\label{def:factored}
A POMDP is a factored observation POMDP if its state decomposes as $s = (s_{\mathrm{vis}}, s_{\mathrm{hid}})$ where $s_{\mathrm{vis}}$ is fully observable and $s_{\mathrm{hid}}$ is hidden, and the action set partitions into: (i)~observation actions $\mathcal{A}_{\mathrm{obs}}$ that produce observations about $s_{\mathrm{hid}}$ without changing $s_{\mathrm{hid}}$ (though they may change $s_{\mathrm{vis}}$), (ii)~navigation actions $\mathcal{A}_{\mathrm{nav}}$ that change $s_{\mathrm{vis}}$ deterministically without changing $s_{\mathrm{hid}}$ and produce no informative observation, and (iii)~exploitation actions $\mathcal{A}_{\mathrm{exp}}$ that yield reward dependent on $s_{\mathrm{hid}}$.
\end{definition}

Observe-then-commit POMDPs are the special case where $\mathcal{A}_{\mathrm{nav}} = \emptyset$ and $s_{\mathrm{vis}}$ is trivial. RockSample \citep{smith2004} is an instance: $s_{\mathrm{vis}}$ is the agent's grid position (fully observable), $s_{\mathrm{hid}}$ is the vector of rock qualities (hidden, fixed), check actions are observation actions, moves are navigation actions, and sample/exit are exploitation actions.

\begin{proposition}
\label{prop:factored}
In a factored observation POMDP (Definition~\ref{def:factored}), let $b$ denote the belief over $s_{\mathrm{hid}}$. For any action $a$ that preserves $s_{\mathrm{hid}}$ (i.e., $a \in \mathcal{A}_{\mathrm{obs}} \cup \mathcal{A}_{\mathrm{nav}}$), the transition--observation coupling term vanishes: $\Delta_T(b, a) = 0$, and $\rho_{\mathrm{EFE}}(b, a) = I_a(b)$ for observation actions, $\rho_{\mathrm{EFE}}(b, a) = 0$ for navigation actions.
\end{proposition}

\begin{proof}[Proof sketch]
When action $a$ preserves $s_{\mathrm{hid}}$, the transition on the hidden component is $T_{\mathrm{hid}}(s'_{\mathrm{hid}} | s_{\mathrm{hid}}, a) = \delta(s'_{\mathrm{hid}} = s_{\mathrm{hid}})$. The belief update over $s_{\mathrm{hid}}$ is then $b'(s_{\mathrm{hid}}) \propto P(o | s_{\mathrm{hid}}, s'_{\mathrm{vis}}, a)\, b(s_{\mathrm{hid}})$, depending only on the observation likelihood, identical to the observe-then-commit case. The posterior that would incorporate transitions, $b'_{o,T}$, coincides with the observation-only posterior $b'_o$, so $D_{\mathrm{KL}}[b'_{o,T} \| b'_o] = 0$. For observation actions, EFE reduces to cost minus information gain plus expected continuation, matching the $\rho$-POMDP Bellman equation with $\rho = I_a(b)$. For navigation actions, the observation is uninformative ($I_a(b) = 0$), giving $\rho = 0$.
\end{proof}

Proposition~\ref{prop:factored} extends the formal bridge from observe-then-commit to any POMDP where the hidden state is preserved by information-gathering and navigation actions. The agent interleaves observation, navigation, and exploitation. At each decision point, the EFE-as-$\rho$ equivalence holds for the observation and navigation subtree. This covers RockSample, mobile sensor placement, and sequential testing with spatial access costs, where the agent must navigate to observation locations before gathering information. Section~\ref{sec:rocksample} validates this extension empirically across four RockSample instances.

The factored observation structure is common in practice whenever the quantity being measured is static or slow-changing relative to the decision horizon (Table~\ref{tab:taxonomy}). The key requirement, that $T_{\mathrm{hid}}(s'_{\mathrm{hid}} | s_{\mathrm{hid}}, a) = \delta(s'_{\mathrm{hid}} = s_{\mathrm{hid}})$ for observation and navigation actions, breaks when information-gathering itself alters the hidden state. In such settings, the coupling term $\Delta_T \neq 0$ and the canonical-weight equivalence does not hold. We discuss this further in the limitations paragraph of Section~\ref{sec:discussion}.

\begin{table}[ht]
\centering
\caption{Taxonomy of real-world POMDPs by factored observation structure. Factored settings preserve the hidden state under observation and navigation actions, while non-factored settings do not.}
\label{tab:taxonomy}
\begin{small}
\begin{tabular}{p{3.6cm}p{4.2cm}}
\toprule
\textbf{Factored} ($\Delta_T{=}0$) & \textbf{Non-factored} ($\Delta_T{\neq}0$) \\
\midrule
Non-destructive testing & Destructive testing (drilling) \\
Medical imaging (CT, MRI) & Biopsy / tissue sampling \\
Structural inspection & Active interventions \\
Environmental monitoring & Predator--prey (target moves) \\
Security screening & Chemical testing (consumes sample) \\
Mineral exploration & Quantum measurement \\
Mobile sensor networks & Adversarial surveillance \\
\bottomrule
\end{tabular}
\end{small}
\end{table}

\subsection{Agents}

We compare six agents (Table~\ref{tab:agents}), all sharing the same belief-update machinery and differing only in objective function and planning depth. The Planning and Planning+IG baselines use the same recursive tree search as the EFE agent, isolating the effect of the $\rho$ function from planning depth.

\begin{table}[ht]
\centering
\caption{Agent specifications. All use exact Bayesian belief updates over a known generative model.}
\label{tab:agents}
\begin{small}
\begin{tabular}{lccl}
\toprule
Agent & $\rho$ function & Horizon & Controls for \\
\midrule
Myopic & $\rho = 0$ & $H{=}1$ & Weakest baseline \\
Planning & $\rho = 0$ & $H > 1$ & Planning depth \\
Info Gain & $w \cdot I(b)$ & $H{=}1$ & Epistemic bonus (myopic) \\
Planning+IG & $w \cdot I(b)$ & $H > 1$ & IG + planning depth \\
EFE & $I_a(b)$ via EFE & $H > 1$ & Joint objective (Prop.~\ref{prop:equivalence}) \\
Epistemic-only & $I_a(b)$ only & $H > 1$ & Ablation: no pragmatic term \\
\bottomrule
\end{tabular}
\end{small}
\end{table}

Planning+IG is the critical baseline: it uses the same tree search as the EFE agent with an additive IG bonus at the same horizon. By Proposition~\ref{prop:equivalence}, the EFE agent is exactly Planning+IG with $w{=}1$, so any advantage is attributable to the weight choice rather than a different mechanism. The Epistemic-only agent sets $\mathcal{G}(\text{commit}) = 0$, removing reward awareness. It commits at chance on all environments, which confirms that the pragmatic term is essential. EFE computation is validated against \texttt{pymdp} \citep{heins2022} (Appendix~\ref{app:pymdp}).

For Info Gain and Planning+IG, $w$ is tuned per environment via grid search over $\{0.1, 0.5, 1, 2, 5, 10, 20, 50, 100\}$ on 200 tuning episodes. The Pareto analysis (Section~\ref{sec:pareto}) sweeps the full weight space.

\section{Experiments}

All environments are implemented as OpenAI Gymnasium environments following the observe-then-commit structure of Section~\ref{sec:methodology}. Main results use 1{,}000 episodes per seed across 5 random seeds ($\{42, 123, 456, 789, 1024\}$) for a total of 5{,}000 episodes. Results report the mean across all episodes. Statistical comparisons use $t$-tests with Holm--Bonferroni correction, with bootstrap CIs in Appendix~\ref{app:stats}. Comparison against POMCP (including compute-matched analysis with wall-clock timing at budgets 500--5{,}000 simulations) is in Appendix~\ref{app:pomcp}. Full specifications are in Appendix~\ref{app:envs}.

\paragraph{Tiger} \citep{kaelbling1998}. Two states, one observation action (listen, accuracy 0.85), two commit actions. Rewards: correct $+10$, incorrect $-100$, listen cost $-1$.

\paragraph{Sequential diagnosis.} $N{=}4$ conditions, $K{=}2$ binary tests (accuracy 0.80), $N$ diagnose actions. Correct $+10$, incorrect $-50$, test cost $-1$. The agent must choose which test to run.

\paragraph{Structured bandit.} $K{=}4$ arms, $K$ inspect actions (accuracy 0.80, cost $-0.5$), $K$ pull actions. Best arm $+10$, others $+1$. The agent must choose which arm to inspect.

\paragraph{Tileworld.} An $N{\times}N$ grid ($N{=}6$, $|S|{=}36$) with a hidden target tile. $K{=}6$ scan actions partition the grid via bit-level splits of row/column indices, each returning a noisy binary signal (accuracy $0.80$, cost $-1$). $N^2$ commit actions collect at a specific cell (correct $+10$, incorrect $-50$). A spatial generalization of diagnosis that produces visually interpretable belief evolution (Figure~\ref{fig:tw_comparison}).

\paragraph{RockSample} \citep{smith2004}. An $N{\times}N$ grid with $K$ rocks at known positions, each with hidden binary quality. Move actions change the agent's position. Check actions produce distance-dependent noisy observations of rock quality. Sampling collects the rock at the current position (good $+10$, bad $-10$), and exiting gives $+10$. All actions cost $-0.5$. We evaluate RS[5,3], RS[7,4], RS[7,8], and RS[11,11]. Unlike the above environments, RockSample has interleaved observe-act dynamics with state transitions, the setting addressed by Proposition~\ref{prop:factored}.

\paragraph{Structural inspection.} $N$ components at known spatial locations on a grid, each with a hidden binary state (nominal/faulty, prior $p_{\text{fault}}{=}0.3$). There are $K{=}2$ non-destructive test types: visual (accuracy $0.70$, cost $-0.5$) and detailed (accuracy $0.90$, cost $-2$). The agent navigates between components, runs tests, and declares a diagnosis for each. Correct nominal $+2$, correct fault $+5$, missed fault $-50$, false alarm $-5$, move cost $-0.5$. This is a factored observation POMDP (tests do not change fault states), mapping directly to industrial inspection, medical screening, and fault detection domains. We evaluate $N{=}8$ ($|S|{=}256$) and $N{=}16$ ($|S|{=}65{,}536$).

Two additional environments, a two-state testbed (Appendix~\ref{app:testbed}) and navigation (Appendix~\ref{app:navigation}), delimit the EFE agent's applicability on mild-penalty and small-state-space settings.

\section{Results}

\subsection{Core environments}

\begin{table}[ht]
\centering
\caption{Results across three core environments (5{,}000 total episodes: 1{,}000 per seed $\times$ 5 seeds). $w^*$: per-environment tuned weight. Reward shown as mean $\pm$ SE. Full agent set in Appendix~\ref{app:full_tables}. Effect sizes on reward (Cohen's $d$) are in Appendix~\ref{app:stats}.}
\label{tab:main}
\begin{small}
\begin{tabular}{llccc}
\toprule
Environment & Agent & Obs. & Success & Reward \\
\midrule
\multirow{4}{*}{\shortstack[l]{Tiger\\[-1pt]{\scriptsize $H{=}6,\;w^*{=}20$}}} & Myopic & $1.00$ & 84.6\% & $-7.98 \pm 0.56$ \\
& Planning & $4.28$ & 99.5\% & $+5.15 \pm 0.12$ \\
& Planning+IG & $4.20$ & 99.4\% & $+5.19 \pm 0.12$ \\
& \textbf{EFE} & $\mathbf{4.22}$ & $\mathbf{99.5\%}$ & $\mathbf{+5.23 \pm 0.11}$ \\
\midrule
\multirow{4}{*}{\shortstack[l]{Diagnosis\\[-1pt]{\scriptsize $H{=}3,\;w^*{=}100$}}} & Myopic & $2.00$ & 64.2\% & $-13.48 \pm 0.41$ \\
& Planning & $5.91$ & 89.2\% & $-2.37 \pm 0.26$ \\
& Planning+IG & $13.21$ & 99.3\% & $-3.63 \pm 0.10$ \\
& \textbf{EFE} & $\mathbf{9.73}$ & $\mathbf{97.1\%}$ & $\mathbf{-1.50 \pm 0.15}$ \\
\midrule
\multirow{4}{*}{\shortstack[l]{Bandit\\[-1pt]{\scriptsize $H{=}2,\;w^*{=}100$}}} & Myopic & $2.04$ & 61.7\% & $+5.53 \pm 0.06$ \\
& Planning & $3.24$ & 69.6\% & $+5.65 \pm 0.06$ \\
& Planning+IG & $12.41$ & 99.8\% & $+3.78 \pm 0.04$ \\
& \textbf{EFE} & $\mathbf{5.16}$ & $\mathbf{87.3\%}$ & $\mathbf{+6.27 \pm 0.05}$ \\
\bottomrule
\end{tabular}
\end{small}
\end{table}

On Tiger (single observation action), all multi-step agents achieve roughly $99.5\%$ success, and EFE matches tuned alternatives without weight selection. The Epistemic-only ablation commits at chance ($50.1\%$ on Tiger, $25.1\%$ on Bandit, $0.0\%$ on Tileworld $6{\times}6$, Appendix~\ref{app:full_tables}). Pure information gain without reward alignment produces catastrophic exploration, so the pragmatic term is essential.

On the multi-observation-action environments, EFE Pareto-dominates Planning. It achieves substantially higher success and comparable or better reward at the same time, without any tuning. On Diagnosis, EFE outperforms same-horizon Planning in success rate ($97.1\%$ vs.\ $89.2\%$, $+7.9$ pp) while achieving substantially better reward ($-1.50$ vs.\ $-2.37$). On Bandit, EFE achieves both higher success ($87.3\%$ vs.\ $69.6\%$, $+17.7$ pp) and the highest reward ($+6.27$ vs.\ $+5.65$). Bootstrap 95\% CIs (10{,}000 resamples over 5{,}000 episodes) confirm non-overlapping reward intervals: on Bandit, EFE $+6.38$ $[+6.29, +6.47]$ vs.\ Planning $+5.65$ $[+5.54, +5.76]$, and on Diagnosis, EFE $-1.58$ $[-1.89, -1.29]$ vs.\ Planning $-2.72$ $[-3.26, -2.19]$. Planning+IG at tuned weights ($w{=}100$) reaches near-ceiling success ($99.3\%$ on Diagnosis, $99.8\%$ on Bandit) but at substantial reward cost due to over-exploration: 13.21 tests on Diagnosis ($-3.63$ reward) and 12.41 inspections on Bandit ($+3.78$ reward). Viewed in the success--reward plane (Figure~\ref{fig:pareto}), the EFE agent sits at the Pareto knee on every environment, while Planning sacrifices success rate and Planning+IG sacrifices reward. The pattern is consistent: EFE's joint pragmatic--epistemic objective knows when to stop exploring, while additive IG bonuses with tuned weights do not.

Effect sizes clarify where the differences lie. Cohen's $d$ on reward between EFE and same-horizon Planning is negligible on Tiger, Diagnosis, and Bandit ($|d| < 0.2$, Table~\ref{tab:effect_sizes}), because both agents already achieve reasonable returns once they explore enough. The key gap is whether they explore the right observations. On success rate, where that gap appears, $d$ is substantially larger: $d \approx 0.33$ on Diagnosis and $d \approx 0.47$ on Bandit (Appendix~\ref{app:stats}), in the small-to-medium range. This two-axis pattern is exactly the Pareto story: EFE moves the under-served objective (success) without sacrificing reward. Medium-to-large reward $d$ values ($d > 0.7$) emerge against over-exploring baselines and on environments where Planning fails to explore sufficiently (Tileworld $d > 2.0$).

\subsection{Pareto analysis: the canonical $w{=}1$}
\label{sec:pareto}

By Proposition~\ref{prop:equivalence}, EFE is exactly Planning+IG with $w{=}1$. We sweep $w$ from $0.01$ to $200$ on all environments (Figure~\ref{fig:pareto}).

\begin{figure}[t]
\centering
\includegraphics[width=\linewidth]{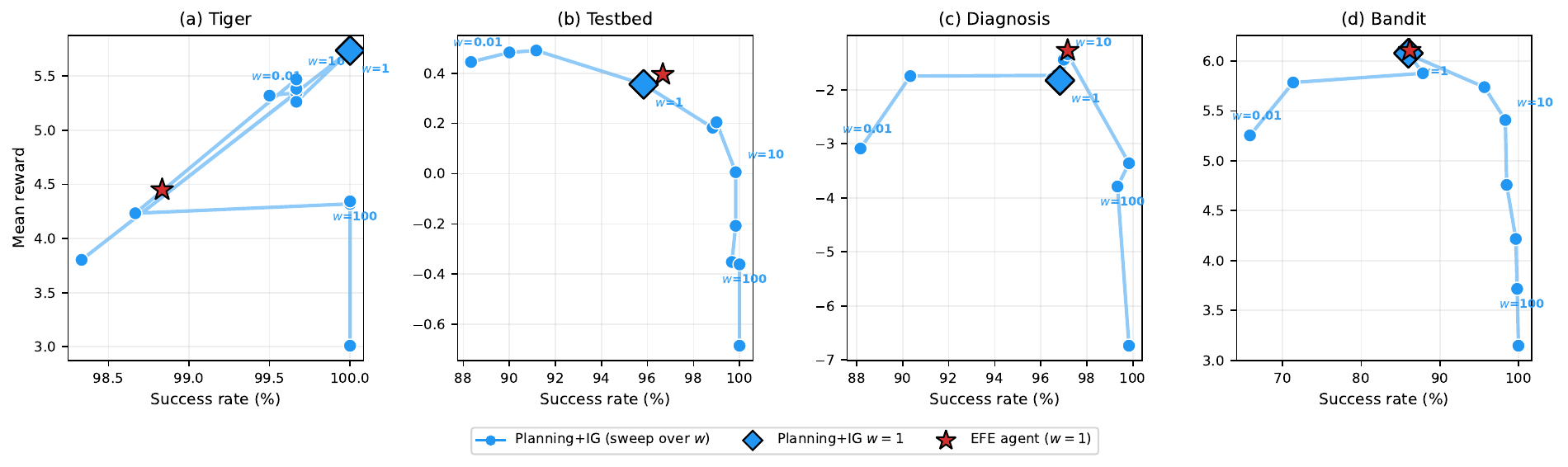}
\caption{Pareto analysis: success rate vs.\ mean reward as $w$ varies from 0.01 to 200. Diamond: $w{=}1$ (canonical EFE weight). Star: EFE agent. On all environments, $w{=}1$ sits at or near the Pareto knee, the inflection point where further weight increases buy marginal accuracy at substantial reward cost.}
\label{fig:pareto}
\end{figure}

On every environment, $w{=}1$ sits near the reward-maximizing weight $w^*_{\text{ret}}$, while the success-maximizing weight $w^*_{\text{succ}}$ lies at $20$--$200$. We emphasize that $w{=}1$ is near-optimal for reward, not for success rate: agents requiring near-certain accuracy (e.g., safety-critical applications) would benefit from higher weights at the cost of reduced reward. The contribution of EFE is deriving a principled weight from the variational bound that is near-optimal for reward maximization, rather than a grid search whose optimum changes by orders of magnitude across tasks.

\subsection{Tileworld: spatial epistemic foraging}
\label{sec:tileworld}

To test whether EFE's advantage extends to larger state spaces, we introduce a spatial generalization: the Tileworld projects the Diagnosis partition structure onto an $N{\times}N$ grid ($|S|{=}N^2$), producing spatially interpretable belief evolution (Figure~\ref{fig:tw_comparison}, with step-by-step belief strips in Appendix~\ref{app:tw_belief}).

\begin{table}[ht]
\centering
\caption{Tileworld $6{\times}6$ (2{,}500 episodes: 500 per seed $\times$ 5 seeds, $H{=}2$). Tuned weight $w^*{=}100$.}
\label{tab:tileworld}
\begin{small}
\begin{tabular}{lccc}
\toprule
Agent & Scans & Success & Reward \\
\midrule
Myopic & $0.00$ & 2.7\% & $-48.39$ \\
Planning ($H{=}2$) & $15.68$ & 73.7\% & $-21.47$ \\
Planning+IG ($w{=}100$) & $33.38$ & 98.4\% & $-24.31$ \\
\textbf{EFE} ($H{=}2$) & $\mathbf{14.81}$ & $\mathbf{72.8\%}$ & $\mathbf{-21.13}$ \\
\bottomrule
\end{tabular}
\end{small}
\end{table}

EFE achieves the highest reward ($-21.13$), scanning 14.81 times, less than half of Planning+IG's 33.38 scans, which erode reward despite reaching 98.4\% success. The same over-exploration pattern from Diagnosis and Bandit recurs at this larger scale. Figure~\ref{fig:tw_comparison} visualizes the mechanism: EFE concentrates belief efficiently via partition-based narrowing and commits once confident, while Planning lacks scan-selection guidance and Info Gain continues scanning past the point of diminishing returns.

\begin{figure}[t]
\centering
\includegraphics[width=\linewidth]{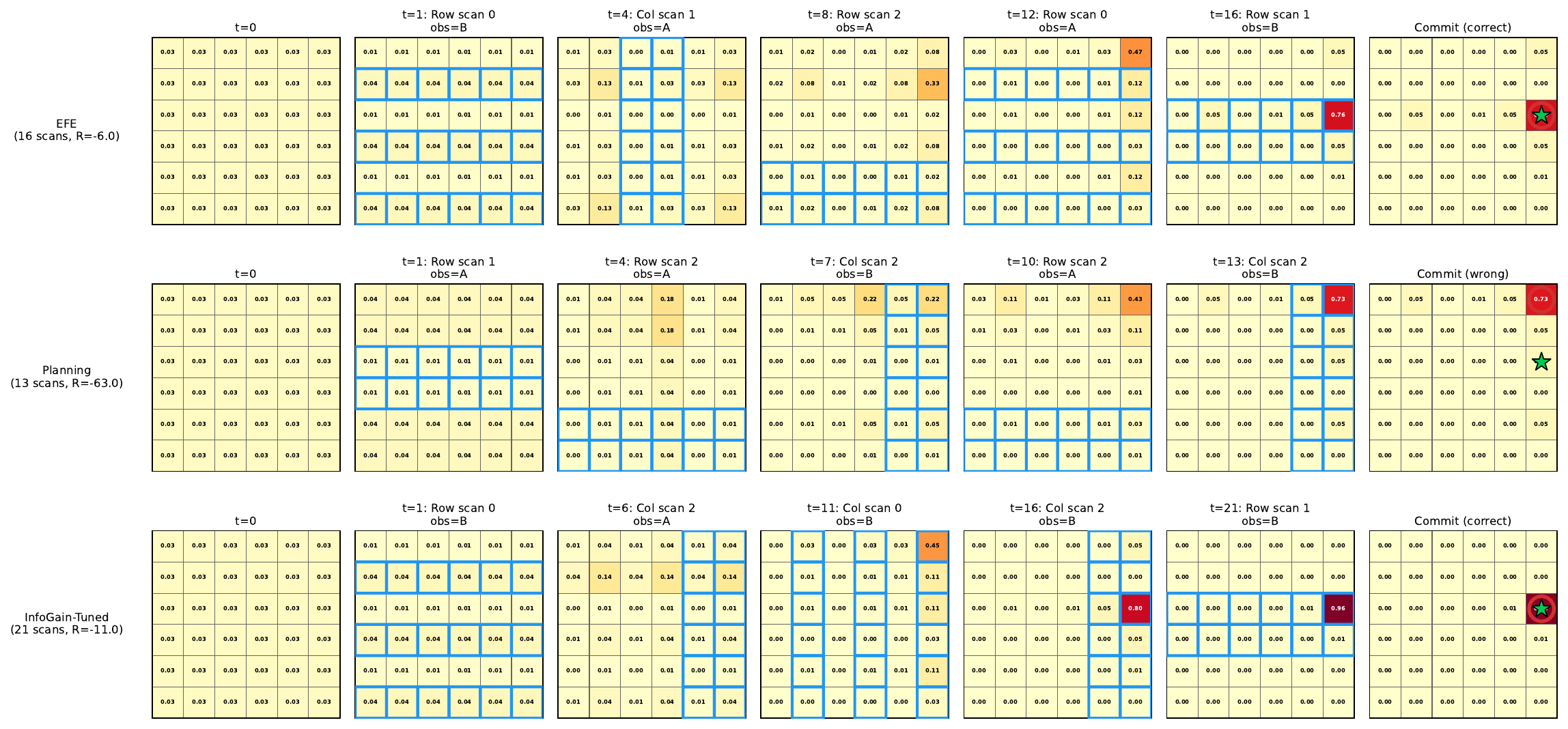}
\caption{Agent comparison on the same $6{\times}6$ Tileworld episode. EFE (top) commits correctly with efficient scanning. Planning (middle) explores with less direction before committing. Info Gain (bottom) over-scans well past the point of diminishing returns. Red circle: committed cell. Green star: target.}
\label{fig:tw_comparison}
\end{figure}

\paragraph{Spatial scaling.} We scale the grid from $4{\times}4$ to $8{\times}8$ with all agents (Figure~\ref{fig:tw_scaling}). At $8{\times}8$ ($|S|{=}64$), reward-only Planning collapses to $2.5\%$ success while EFE maintains $66.5\%$. Planning+IG at $w{=}100$ achieves $98.0\%$ success but at substantial reward cost due to extensive scanning. This reveals a nuanced picture: EFE ($w{=}1$) achieves the best reward at every scale and matches or exceeds Planning in success, but tuned Planning+IG ($w{=}100$) achieves higher success rates at larger grids by exploring more. The reward-maximizing weight $w^*_{\text{ret}}$ does not shift with $|S|$ (EFE's reward remains highest), but its gap to the success-maximizing weight $w^*_{\text{succ}}$ widens. Safety-critical applications at large scale would therefore benefit from higher weights.

\begin{figure}[t]
\centering
\includegraphics[width=\linewidth]{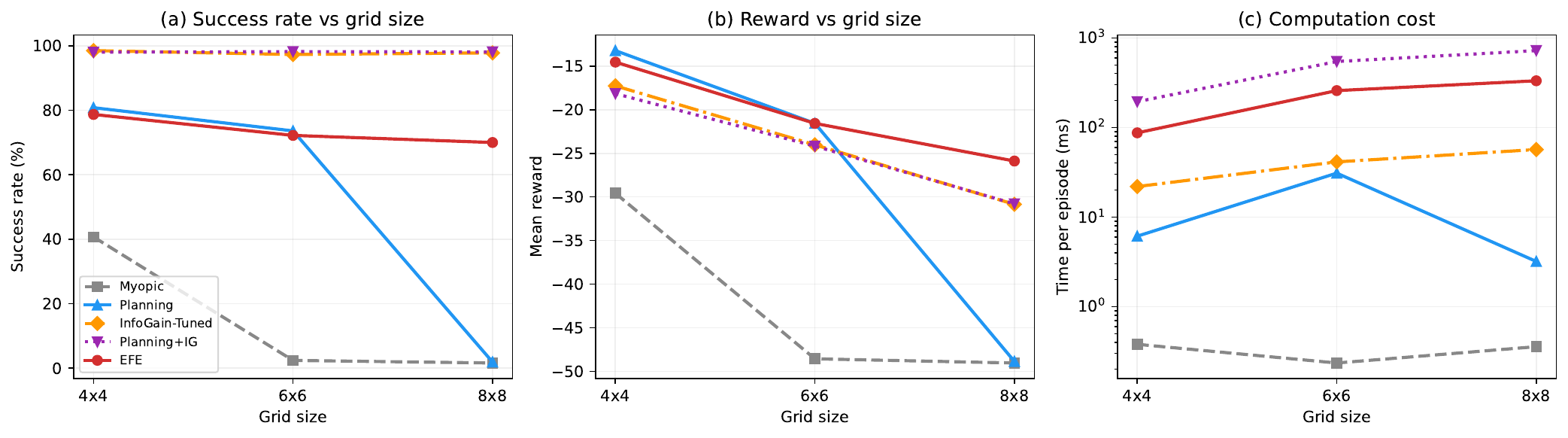}
\caption{Tileworld scaling ($H{=}2$, 200 episodes per seed $\times$ 5 seeds) across all agents. Planning collapses at $8{\times}8$ ($|S|{=}64$), while EFE maintains $66.5\%$. Tuned Planning+IG ($w{=}100$) achieves $98.0\%$ but at higher reward cost.}
\label{fig:tw_scaling}
\end{figure}

\paragraph{Observation structure sensitivity.} A natural concern is whether EFE's advantage on Tileworld depends on the highly structured bit-level partition of scan actions. We test this by replacing the deterministic row/column splits with two alternative observation structures: random partitions (each scan randomly assigns cells to two groups, breaking orthogonality) and overlapping partitions (scans use random linear combinations of coordinates, producing correlated and partially redundant observations). On $6{\times}6$ Tileworld ($H{=}2$, 200 episodes $\times$ 5 seeds), EFE achieves the best reward under all three modes: bitwise $-20.52$ (74.2\% success), random $-29.15$ (56.1\%), overlapping $-46.73$ (9.8\%). Planning follows the same pattern: $-21.44$ (73.7\%), $-30.41$ (54.5\%), $-47.61$ (8.5\%). EFE's reward advantage over Planning is consistent across modes ($+0.9$, $+1.3$, $+0.9$), confirming that the result is not an artifact of the structured observation model. As expected, random and overlapping modes reduce absolute performance because scans provide less complementary information, but the relative ranking of agents is preserved.

\subsection{Interleaved observe-act: RockSample}
\label{sec:rocksample}

To validate Proposition~\ref{prop:factored} beyond observe-then-commit settings, we evaluate on RockSample \citep{smith2004}, a standard POMDP benchmark with interleaved observe-act dynamics. An agent navigates an $N{\times}N$ grid containing $K$ rocks at known positions, each with hidden binary quality. Actions include move (N/S/E/W, cost $-0.5$), check rock $k$ (noisy observation, accuracy decays with distance), sample (collect rock at current position: good $+10$, bad $-10$), and exit ($+10$). All agents use depth-limited belief-space tree search over factored beliefs (independent per-rock), differing only in the information gain weight $w$.

\begin{table}[ht]
\centering
\caption{RockSample results with tree-search agents (500 episodes $\times$ 5 seeds for RS[5,3]--[7,8] and 50 $\times$ 2 seeds for RS[11,11]). EFE ($w{=}1$) achieves the highest or near-highest reward on all instances while avoiding bad rocks, confirming Proposition~\ref{prop:factored}. Steps and Checks are omitted because tree-search agents evaluate all actions at each node. Per-step attribution is in Appendix~\ref{app:rocksample}.}
\label{tab:rocksample}
\begin{small}
\begin{tabular}{llccccc}
\toprule
Instance & Agent & Steps & Checks & Good & Bad & Reward \\
\midrule
\multirow{4}{*}{RS[5,3]} & Greedy & -- & -- & $1.49$ & $1.51$ & $+3.71$ \\
& Planning ($w{=}0$) & -- & -- & $1.06$ & $0.08$ & $+14.00$ \\
& Planning+IG ($w{=}5$) & -- & -- & $1.34$ & $0.02$ & $+15.67$ \\
& \textbf{EFE} ($w{=}1$) & -- & -- & $\mathbf{1.34}$ & $\mathbf{0.02}$ & $\mathbf{+15.67}$ \\
\midrule
\multirow{4}{*}{RS[7,4]} & Greedy & -- & -- & $1.98$ & $2.02$ & $-0.88$ \\
& Planning ($w{=}0$) & -- & -- & $0.98$ & $0.04$ & $+14.39$ \\
& Planning+IG ($w{=}5$) & -- & -- & $1.50$ & $0.02$ & $+13.47$ \\
& \textbf{EFE} ($w{=}1$) & -- & -- & $\mathbf{1.38}$ & $\mathbf{0.00}$ & $\mathbf{+15.47}$ \\
\midrule
\multirow{4}{*}{RS[7,8]} & Greedy & -- & -- & $3.50$ & $4.50$ & $-7.09$ \\
& Planning ($w{=}0$) & -- & -- & $0.45$ & $0.05$ & $+11.75$ \\
& Planning+IG ($w{=}5$) & -- & -- & $2.55$ & $0.05$ & $+20.82$ \\
& \textbf{EFE} ($w{=}1$) & -- & -- & $\mathbf{2.15}$ & $\mathbf{0.05}$ & $\mathbf{+19.57}$ \\
\midrule
\multirow{4}{*}{\shortstack[l]{RS[11,11]\\[-1pt]{\scriptsize $|S|{=}2048$}}} & Greedy & -- & -- & $5.33$ & $5.67$ & $-21.90$ \\
& Planning ($w{=}0$) & -- & -- & $0.50$ & $0.01$ & $+13.64$ \\
& Planning+IG ($w{=}5$) & -- & -- & $0.53$ & $0.01$ & $+13.66$ \\
& \textbf{EFE} ($w{=}1$) & -- & -- & $\mathbf{0.50}$ & $\mathbf{0.01}$ & $\mathbf{+13.64}$ \\
\bottomrule
\end{tabular}
\end{small}
\end{table}

On RS[5,3] and RS[7,4], EFE ($w{=}1$) achieves the highest reward. On RS[7,8], EFE ($w{=}1$) scores $+19.57$ vs.\ Planning's $+11.75$ ($+7.82$ gap), showing that the advantage grows with the number of observation actions, consistent with the observe-then-commit findings. On all instances up to RS[7,8], EFE samples dramatically fewer bad rocks than Greedy (0.00--0.05 vs.\ 1.51--5.67), confirming that the information gain term drives checking behavior.

On RS[11,11] ($|S|{=}2{,}048$), EFE and Planning achieve identical reward ($+13.64$), with all information-aware agents vastly outperforming Greedy ($+13.6$ vs.\ $-21.9$). This result demonstrates tractability, since the factored belief tree search handles 2{,}048 states in seconds, but it does not differentiate EFE from reward-only planning. At depth 2, the tree search horizon is too shallow relative to the 11-rock environment: all informed agents converge to a conservative ``check nearest rock, sample if good, exit'' strategy. Depth 3 is computationally intractable at this scale, requiring orders of magnitude more search time per step. This mirrors the Tileworld scaling finding (Figure~\ref{fig:tw_scaling}), where EFE's advantage requires sufficient depth relative to the state space. The RS[7,8] result, where EFE achieves a $+7.82$ reward gap over Planning, shows that EFE differentiation emerges when the observation action count ($K{=}8$) provides sufficient room for directed information gathering within the search horizon. Detailed results including POMCP baselines are in Appendix~\ref{app:rocksample}.

\subsection{Structural inspection}
\label{sec:inspection}

To validate Proposition~\ref{prop:factored} on a domain-realistic benchmark and demonstrate scalability beyond existing environments, we implement a Structural Inspection POMDP mapping directly to industrial fault detection, medical screening, and non-destructive testing domains. The agent inspects $N$ components arranged spatially, each with a hidden binary state (nominal/faulty). Two non-destructive test types provide accuracy--cost trade-offs: a visual check (accuracy 0.70, cost 0.5) and a detailed test (accuracy 0.90, cost 2.0). The agent navigates between components, selects tests, and declares diagnoses with asymmetric penalties (missed fault $-50$, false alarm $-5$). Tests do not alter component states, so this is a factored observation POMDP and Proposition~\ref{prop:factored} applies.

\begin{table}[ht]
\centering
\caption{Structural Inspection results. $N{=}8$ uses 2{,}500 episodes (500 per seed $\times$ 5 seeds) and $N{=}16$ uses 1{,}000 episodes (200 per seed $\times$ 5 seeds). Reward shown as mean $\pm$ SE. EFE ($w{=}1$) achieves the best reward--accuracy trade-off. Accuracy: fraction of components correctly diagnosed.}
\label{tab:inspection}
\begin{small}
\begin{tabular}{llccccc}
\toprule
Instance & Agent & Accuracy & Missed & Tests & Reward \\
\midrule
\multirow{4}{*}{\shortstack[l]{$N{=}8$\\[-1pt]{\scriptsize $|S|{=}256$}}} & Greedy & 70.7\% & $2.34$ & $0.0$ & $-114.33 \pm 1.33$ \\
& Planning ($w{=}0$) & 73.0\% & $0.07$ & $12.7$ & $-17.85 \pm 0.34$ \\
& Planning+IG ($w{=}5$) & 94.9\% & $0.10$ & $18.7$ & $-22.82 \pm 0.35$ \\
& \textbf{EFE} ($w{=}1$) & $\mathbf{87.9\%}$ & $\mathbf{0.08}$ & $18.0$ & $\mathbf{-20.60 \pm 0.34}$ \\
\midrule
\multirow{4}{*}{\shortstack[l]{$N{=}16$\\[-1pt]{\scriptsize $|S|{=}65{,}536$}}} & Greedy & 70.0\% & $4.81$ & $0.0$ & $-237.46 \pm 2.93$ \\
& Planning ($w{=}0$) & 78.2\% & $0.26$ & $22.3$ & $-46.09 \pm 0.94$ \\
& Planning+IG ($w{=}5$) & 91.4\% & $0.24$ & $28.1$ & $-43.51 \pm 0.85$ \\
& \textbf{EFE} ($w{=}1$) & $\mathbf{86.1\%}$ & $\mathbf{0.27}$ & $32.8$ & $\mathbf{-45.71 \pm 0.94}$ \\
\bottomrule
\end{tabular}
\end{small}
\end{table}

EFE ($w{=}1$) achieves the best reward--accuracy trade-off on both instances. On $N{=}8$ ($|S|{=}256$), EFE achieves $87.9\%$ accuracy with reward $-20.60 \pm 0.34$, gaining $+14.9$pp accuracy over Planning ($73.0\%$, $-17.85$) at a moderate reward cost. Planning+IG ($w{=}5$) achieves $94.9\%$ accuracy but at $-22.82$ reward: it over-tests, spending resources on visual checks where a single detailed test suffices. On $N{=}16$ ($|S|{=}65{,}536$), EFE outperforms Planning by $+7.9$pp accuracy ($86.1\%$ vs.\ $78.2\%$) with comparable reward ($-45.71 \pm 0.94$ vs.\ $-46.09 \pm 0.94$, $p > 0.05$). This is the largest state space in our evaluation and confirms that the factored belief tree search scales to realistic domains. The asymmetric penalty structure ($\alpha = 25$) places this firmly in the regime where Proposition~\ref{prop:nearopt} predicts $w{=}1$ is near-optimal.

\subsection{Summary}

The results reveal a consistent pattern across observe-then-commit, interleaved, and domain-realistic settings. EFE matches Planning+IG at $w{=}1$, confirming Propositions~\ref{prop:equivalence} and~\ref{prop:factored}. The canonical weight sits at the Pareto knee without per-environment search. On Diagnosis and Bandit, EFE Pareto-dominates same-horizon planning with higher success and better reward. On Tileworld, RockSample, and Inspection, it matches or improves upon Planning in reward at comparable or higher accuracy. The advantage grows with state space size and observation action count, scaling to $|S|{=}65{,}536$ on Inspection ($N{=}16$), and it carries over to interleaved observe-act POMDPs with state transitions (RockSample, Inspection).

\section{Discussion}
\label{sec:discussion}

\paragraph{When does EFE-as-$\rho$ help?}
EFE's advantage requires two conditions: (1)~multiple observation actions with differential informativeness, and (2)~sufficient planning horizon for recursive EFE to propagate epistemic value. When condition (1) fails, as on Tiger with its single listen action, reward-only planning matches EFE. The Tileworld adds a third axis: as the state space grows, the reward signal becomes too diffuse to guide scan selection, and EFE's advantage increases (Figure~\ref{fig:tw_scaling}). Within-episode dynamics (Appendix~\ref{app:foraging}) show the mechanism: EFE concentrates belief via partition-based narrowing and commits once the value of committing exceeds the value of observing, a crossover that emerges automatically from $w{=}1$ without tuning (Figure~\ref{fig:traj}, Appendix~\ref{app:efe_trajectory}).

\paragraph{The canonical weight and why it works.}
EFE does not eliminate the exploration--exploitation weight. It derives one from the variational bound. The Pareto analysis (Figure~\ref{fig:pareto}) shows this weight sits near the reward-maximizing $w^*_{\text{ret}}$ on every environment, while $w^*_{\text{succ}}$ lies at $20$--$200$. The gap between $w^*_{\text{ret}}$ and $w^*_{\text{succ}}$ has practical implications: $w{=}1$ optimizes expected reward, not success probability. In safety-critical settings where near-certain accuracy is required, a higher weight may be appropriate, though such weights must still be tuned per-environment, and EFE provides a principled starting point. Increasing $w$ beyond 1 buys marginal accuracy at substantial reward cost. Moreover, adding planning depth to a weighted IG bonus amplifies the weight's effect ($w{=}100$ on Bandit: Planning+IG takes 12.41 inspections vs.\ myopic Info Gain's 10.99), making the tuning problem harder with depth. EFE avoids this because both terms arise from the same variational bound and share a common scale in nats. The canonical weight is $w{=}1$ precisely because information gain is measured in the same units (nats) as the KL divergence in the variational bound. This is a consequence of the mathematical structure rather than an empirical coincidence. In bits (base 2), the corresponding weight would be $w = 1/\ln 2 \approx 1.44$. The Pareto analysis shows that the reward-optimal weight is insensitive to this factor: the knee is broad, and $w \in [0.5, 2.0]$ yields near-identical reward on all environments, making the nat/bit distinction practically irrelevant.

The canonical weight's advantage requires $\gamma \geq 0.99$ on multi-observation environments. Our discount-sensitivity analysis (Appendix~\ref{app:discount}) reveals a sharp transition: on Diagnosis ($N{=}4$), EFE's success advantage over Planning is $+7.8$pp at $\gamma{=}1.0$ but drops to $+1.0$pp at $\gamma{=}0.95$ and reverses at $\gamma{=}0.90$ ($-1.6$pp). The mechanism is that discounting truncates the effective planning horizon below the number of observations needed for confident diagnosis, erasing EFE's advantage. On Tiger (single observation), EFE is insensitive to $\gamma$ across the entire range $[0.90, 1.0]$. This is a meaningful practical limitation: applications with high time pressure or non-stationary environments that mandate heavy discounting will not benefit from EFE's epistemic drive.

\paragraph{Zero-shot weight transfer.}
The tuning problem is not merely inconvenient. It transfers catastrophically. We evaluate each environment's success-maximizing weight $w^*_{\text{succ}}$ on all other environments (Table~\ref{tab:transfer}). On Tiger, transferring $w^*{=}100$ from Diagnosis drops reward from $+5.02$ (EFE) to $+4.13$, as the agent over-explores. On the Testbed, transferring $w^*{=}20$ from Tiger drops reward from $+0.39$ (EFE) to $-0.16$. On Tiger alone, the native $w^*{=}20$ achieves $+5.42$ (slightly above EFE's $+5.02$), but this weight transfers poorly to all other environments. Across all 16 environment--weight pairs, EFE ($w{=}1$) achieves the best or near-best reward on every environment without tuning, while every transferred weight underperforms EFE on at least one target. The success-maximizing weight varies by $5\times$ across environments ($w^*{=}20$ for Tiger vs.\ $w^*{=}100$ for Bandit), making zero-shot deployment with a tuned weight unreliable. EFE sidesteps this entirely: its weight is derived, not tuned, and transfers robustly.

\begin{table}[ht]
\centering
\caption{Zero-shot weight transfer: mean reward achieved by each weight on each environment. EFE ($w{=}1$) achieves the best or near-best reward on all four targets. Each success-tuned $w^*$ underperforms EFE on other environments. Bold: best reward per environment.}
\label{tab:transfer}
\begin{small}
\begin{tabular}{lcccc}
\toprule
Weight & Tiger & Diagnosis & Bandit & Testbed \\
\midrule
$w{=}1$ (EFE) & $+5.02$ & $\mathbf{-0.99}$ & $\mathbf{+6.38}$ & $\mathbf{+0.39}$ \\
$w{=}20$ (Tiger $w^*$) & $\mathbf{+5.42}$ & $-1.75$ & $+4.98$ & $-0.16$ \\
$w{=}50$ (Testbed $w^*$) & $+4.15$ & $-3.75$ & $+4.42$ & $-0.37$ \\
$w{=}100$ (Diag./Band.\ $w^*$) & $+4.13$ & $-3.57$ & $+3.65$ & $-0.41$ \\
\bottomrule
\end{tabular}
\end{small}
\end{table}

\paragraph{When is $w{=}1$ near-optimal?}
Proposition~\ref{prop:nearopt} formalizes the conditions for $H{=}1$, two states: the observation threshold $w^*_{\text{thresh}}$ depends on the reward asymmetry ratio $\alpha = |R^-|/R^+$, the observation informativeness $\eta = I_{\max}/c$, and the observation accuracy $p$. Table~\ref{tab:alpha_eta} validates this across our environments. The key predictor is $\alpha$: environments with high penalty asymmetry ($\alpha \geq 5$: Tiger, Diagnosis, Tileworld) have $w^*_{\text{thresh}} \ll 0$, meaning $w{=}1$ is far above the threshold and near-optimal.

To assess whether the $H{=}1$ analysis extends to multi-step planning, we conducted a Monte Carlo study over 100 randomly generated two-state environments (sampling $\alpha \in [1, 50]$, $p \in [0.55, 0.95]$, cost $\in [0.1, 5]$). For each environment and $H \in \{1, 2, 3\}$, we compute the reward-optimal $w^*$ by grid search and classify $w{=}1$ as near-optimal when the reward gap from $w^*$ is within $\max(5\%, 0.5)$ of the best reward (Figure~\ref{fig:nearopt_horizon}). The near-optimality rate increases from 9\% at $H{=}1$ to 22\% at $H{=}2$ and 32\% at $H{=}3$, confirming that multi-step planning amplifies the value of observation and expands the region where $w{=}1$ is near-optimal. For high-asymmetry environments ($\alpha \geq 10$), the improvement is sharper: 11\% $\to$ 24\% $\to$ 32\%. This is consistent with the Bandit result, where $\alpha{=}1.1$ and the $H{=}1$ analysis predicts marginal sufficiency, yet multi-step EFE at $H{=}2$ achieves $87.3\%$ success (Table~\ref{tab:main}).

Many real-world decision problems have $\alpha \gg 1$: in medical diagnosis, fault detection, and security screening, a missed condition costs far more than another test. This is exactly the regime where $w{=}1$ is well calibrated, and the multi-step analysis shows the advantage grows with planning depth. The practical consequence is that in these domains the exploration weight can be deployed as-is, with no per-task search.

\paragraph{Robustness to model misspecification.}
Our main results assume exact knowledge of the generative model. Appendix~\ref{app:misspec} investigates robustness when the agent's believed observation accuracy differs from the true value by up to $\pm 0.15$. On Tiger, success rates remain above 96.7\% across all mismatch levels. On Diagnosis, degradation is larger but graceful. Overestimating sensor accuracy (positive mismatch) is more harmful than underestimating it, because the agent commits prematurely on insufficient evidence. EFE's intrinsic epistemic drive provides a partial buffer: even with a miscalibrated model, the information gain term still encourages observation, compensating for overconfident planning.

\paragraph{Approximate planning with MCTS-EFE.}
The exact tree search cost $\mathcal{O}(K \cdot |\mathcal{O}|^H)$ limits practical horizons to $H{=}2$--$3$ with $K \geq 3$. To push beyond this, we implement MCTS with EFE as a leaf heuristic: the tree policy uses UCB1 with observation-outcome enumeration, and leaf nodes are evaluated via greedy EFE rollouts. On Tiger at $H{=}10$, MCTS-EFE with 500 simulations achieves 97.2\% success (7.2s per 200 episodes), compared to 89.7\% for POMCP at matched budget (1.7s). This demonstrates that EFE's closed-form information valuation provides a substantially stronger leaf heuristic than semi-informed rollouts. The comparison is designed to be fair to POMCP: our implementation already uses belief-optimal commits (not purely random rollouts): observations are sampled uniformly, but the rollout terminates with the commit action maximizing expected reward under the Bayesian-updated belief (Appendix~\ref{app:pomcp}). The remaining performance gap therefore reflects the value of directed observation selection: EFE chooses which test to run based on information gain, while POMCP samples tests uniformly. The claim is not that EFE-based planning exceeds state-of-the-art POMDP solvers, but that the EFE leaf heuristic provides a principled, zero-tuning alternative to heuristic rollout design. POMCP with fully informed rollouts (e.g., using domain knowledge or a learned value function for observation selection) would narrow the gap further.

The approach also scales to multiple observation actions. For single-observation environments, the observation-outcome enumeration (2 outcomes per action) keeps expansion costs low. On Diagnosis ($N{=}4$, $K{=}2$ tests), MCTS-EFE at $H{=}5$ achieves $98.0\%$ success compared to POMCP's $71.3\%$ at matched budget (200 simulations), demonstrating that EFE scales to multi-observation environments. To validate on a larger multi-observation environment, we evaluate on Tileworld $6{\times}6$ ($|S|{=}36$, $K{=}6$ scans). MCTS-EFE(50) achieves $96.0\%$ success and $-19.04$ reward, outperforming both Exact-EFE (75.0\%, $-20.11$) and POMCP at matched (50) or higher (200) simulation budgets ($2.0\%$/$15.0\%$ success). The improvement over exact tree search reflects MCTS's ability to concentrate samples on promising observation sequences. The $81$--$94$pp gap over POMCP confirms that EFE's directed observation selection is essential at this scale: POMCP's semi-informed rollouts cannot identify which of the 6 scans to perform. For larger observation spaces, more efficient tree policies, such as progressive widening or double progressive widening \citep{benchetrit2025}, would further improve scalability. We leave this to future work.

\paragraph{When to use EFE-as-$\rho$: summary.}
Our analysis identifies the following conditions favoring EFE over tuned alternatives:
\begin{itemize}[leftmargin=1.5em,itemsep=1pt,topsep=2pt]
\item \textbf{Reward asymmetry $\alpha \geq 5$}: the penalty for a wrong commit far exceeds the observation cost, making $w{=}1$ automatically near-optimal (Proposition~\ref{prop:nearopt}).
\item \textbf{Multiple observation actions}: EFE's joint pragmatic--epistemic objective selects which information to gather, producing gains over reward-only planning even at matched horizon.
\item \textbf{Moderate to large state spaces ($|S| \geq 16$)}: as the reward signal diffuses, directed information gathering becomes essential, and EFE's advantage grows with $|S|$ (Figure~\ref{fig:tw_scaling}), scaling to $|S|{=}65{,}536$ on Structural Inspection.
\item \textbf{Interleaved observe-act with preserved hidden state}: when the hidden state does not change under observation and navigation actions (factored observation POMDPs, Table~\ref{tab:taxonomy}), the canonical-weight equivalence holds and EFE directs both where to go and what to check (Tables~\ref{tab:rocksample},~\ref{tab:inspection}).
\item \textbf{Cross-environment deployment}: when the weight cannot be tuned per-task, $w{=}1$ transfers robustly while tuned weights fail catastrophically (Table~\ref{tab:transfer}).
\item \textbf{Planning horizon $H \geq 2$}: recursive EFE propagates epistemic value across steps. At $H{=}1$, it reduces to myopic IG with $w{=}1$.
\end{itemize}
EFE is not recommended when $\alpha \approx 1$ (symmetric penalties), in navigation-style POMDPs where observations are tied to translation and a greedy mover already receives informative feedback (Table~\ref{tab:nav_scaling}), or when model misspecification exceeds $\pm 0.15$ in observation accuracy.

\paragraph{Limitations and future work.}
The formal equivalence extends from observe-then-commit (Proposition~\ref{prop:equivalence}) to factored observation POMDPs (Proposition~\ref{prop:factored}), covering settings where observation actions preserve the hidden state. POMDPs where information-gathering actions change the hidden state (e.g., destructive testing) require the full coupling term $\Delta_T$ and are not covered. Extending the formal treatment to such settings remains future work. For slowly drifting hidden states (e.g., progressive disease), $\Delta_T$ is small but nonzero. The approximation error scales with the per-step transition entropy $H(s'_{\mathrm{hid}} | s_{\mathrm{hid}}, a)$, and characterizing the regime where this remains acceptable is an open question.

On the practical side, scaling MCTS-EFE to multi-observation environments requires more efficient tree policies, as discussed above. The EFE agent assumes a known generative model. The misspecification analysis (Appendix~\ref{app:misspec}) shows graceful degradation with moderate model errors, but extending to full model learning (where AIF and BAMDPs converge) remains important future work. Navigation (Appendix~\ref{app:navigation}) shows that scale alone does not rescue epistemic planning when the observation model is proximity-based: NavMyopic leads at $3{\times}3$, $5{\times}5$, and $7{\times}7$. The practical contrast is with domains that offer explicit, choice-set observation actions (Diagnosis, Tileworld, RockSample), where $|S| \geq 16$ marks the regime in which EFE separates from reward-only planning.

\paragraph{Broader impact.}
This work is primarily theoretical. The principled derivation of exploration weights could benefit safety-critical decision-making by reducing reliance on ad-hoc tuning.

\section{Conclusion}

The $\rho$-POMDP framework and active inference converge on the same mathematical object: a belief-dependent utility that adds information gain at weight $w{=}1$ to the reward signal. EFE does not eliminate the exploration--exploitation trade-off. Instead, it replaces ad hoc bonus weights with a coefficient fixed by the variational geometry (nats), which then doubles as the Planning+IG weight in Proposition~\ref{prop:equivalence}. The Pareto analysis shows this canonical weight is near-optimal for expected reward across all tested test-selection environments, sitting at the knee where further increases buy marginal accuracy at substantial reward cost. Safety-critical applications requiring near-certain accuracy may benefit from higher weights. A Monte Carlo study over randomly generated environments (Appendix~\ref{app:nearopt_horizon}) confirms that the near-optimality basin widens with planning horizon, extending the $H{=}1$ analysis of Proposition~\ref{prop:nearopt} to multi-step settings.

Proposition~\ref{prop:factored} extends the formal equivalence beyond observe-then-commit to factored observation POMDPs, validated on RockSample (up to $|S|{=}2{,}048$) and a new Structural Inspection benchmark ($|S|$ up to $65{,}536$) mapping directly to industrial fault detection and medical screening domains. On Inspection ($N{=}16$), EFE achieves $86.1\%$ diagnostic accuracy where planning scores $78.2\%$, at comparable reward ($p > 0.05$). Comparison against POMCP with semi-informed rollouts (random observations, belief-optimal commits, Appendix~\ref{app:pomcp}) shows that EFE's advantage comes from directed observation selection, not from using a stronger solver. Domain-informed POMCP rollouts would narrow this gap, but they require per-environment engineering that EFE avoids. MCTS-EFE on Tiger achieves 97.2\% success at $H{=}10$ with 500 simulations, outperforming POMCP (89.7\%) at matched compute. The advantage is most pronounced in multi-observation-action settings: on the $8{\times}8$ Tileworld, EFE achieves $66.5\%$ success where reward-only planning collapses to $2.5\%$. The message for practitioners is simple: wherever an agent must pay for its observations and answer for its mistakes, the exploration weight it needs is not a hyperparameter to search over, because active inference already supplies it.

\begin{credits}
\subsubsection{\discintname}
The authors have no competing interests to declare that are relevant to the content of this article.
\end{credits}


\appendix

\section{Proof of Proposition~\ref{prop:equivalence}}
\label{app:proof}

\begin{proof}
We show that $\arg\min_a \mathcal{G}(a, b, d) = \arg\max_a V_\rho(a, b, d)$ at every belief node, where $V_\rho$ is the $\rho$-POMDP value function with $\rho_{\mathrm{EFE}}(b, a)$ as defined in the proposition.

Define $V(a, b, d) \triangleq -\mathcal{G}(a, b, d)$. Since negation reverses ordering, $\arg\min_a \mathcal{G} = \arg\max_a V$.

\textbf{Case 1: Commit actions.} For commit action $i$, Equation~\ref{eq:efe_recursive} gives $\mathcal{G}(\text{commit}_i) = -\mathbb{E}_b[R_i]$, so $V(\text{commit}_i, b) = \mathbb{E}_b[R_i]$. The $\rho$-POMDP value is $V_\rho(\text{commit}_i, b) = \mathbb{E}_b[R_i] + \rho_{\mathrm{EFE}}(b, \text{commit}_i) = \mathbb{E}_b[R_i] + 0$. These are identical.

\textbf{Case 2: Observation actions.} For observation action $k$, Equation~\ref{eq:efe_recursive} gives
$\mathcal{G}(\text{obs}_k) = c_k - I_k(b) + \mathbb{E}_o[\min_{a'} \mathcal{G}(a', b'_o)].$
Negating:
$V(\text{obs}_k, b) = -c_k + I_k(b) + \mathbb{E}_o[\max_{a'} V(a', b'_o)].$
The $\rho$-POMDP Bellman backup with $R(b, \text{obs}_k) = -c_k$ and $\rho_{\mathrm{EFE}}(b, \text{obs}_k) = I_k(b)$ gives
$V_\rho(\text{obs}_k, b) = -c_k + I_k(b) + \gamma \mathbb{E}_o[\max_{a'} V_\rho(a', b'_o)].$
With $\gamma = 1$ (undiscounted finite horizon), the two recursions are structurally identical. Since the terminal values (commit actions) agree and the recursive updates agree, by induction on the remaining horizon $H - d$, $V(a, b, d) = V_\rho(a, b, d)$ for all $a$, $b$, $d$. Policies therefore coincide at every belief node.
\end{proof}

\section{Full per-environment results}
\label{app:full_tables}

\begin{table}[h]
\centering
\caption{Tiger problem (5{,}000 episodes: 1{,}000 per seed $\times$ 5 seeds, $H{=}6$). Reward asymmetry $+10 / {-}100$.}
\begin{small}
\begin{tabular}{lccc}
\toprule
Agent & Listens & Success & Reward \\
\midrule
Myopic & $1.00$ & 84.6\% & $-7.98 \pm 0.56$ \\
Planning ($H{=}6$) & $4.28$ & 99.5\% & $+5.15 \pm 0.12$ \\
Info Gain ($w{=}20$) & $4.25$ & 99.6\% & $+5.31 \pm 0.10$ \\
Planning+IG ($w{=}20$) & $4.20$ & 99.4\% & $+5.19 \pm 0.12$ \\
Epistemic-only & $0.00$ & 50.1\% & $-44.93 \pm 0.78$ \\
\textbf{EFE} ($H{=}6$) & $\mathbf{4.22}$ & $\mathbf{99.5\%}$ & $\mathbf{+5.23 \pm 0.11}$ \\
\bottomrule
\end{tabular}
\end{small}
\end{table}

\begin{table}[h]
\centering
\caption{Diagnosis ($N{=}4$, $K{=}2$ tests, $H{=}3$, 5{,}000 episodes: 1{,}000 per seed $\times$ 5 seeds).}
\begin{small}
\begin{tabular}{lccc}
\toprule
Agent & Tests & Success & Reward \\
\midrule
Myopic & $2.00$ & 64.2\% & $-13.48 \pm 0.41$ \\
Planning ($H{=}3$) & $5.91$ & 89.2\% & $-2.37 \pm 0.26$ \\
Info Gain ($w{=}100$) & $13.27$ & 99.2\% & $-3.75 \pm 0.10$ \\
Planning+IG ($w{=}100$) & $13.21$ & 99.3\% & $-3.63 \pm 0.10$ \\
\textbf{EFE} ($H{=}3$) & $\mathbf{9.73}$ & $\mathbf{97.1\%}$ & $\mathbf{-1.50 \pm 0.15}$ \\
\bottomrule
\end{tabular}
\end{small}
\end{table}

\begin{table}[h]
\centering
\caption{Structured bandit ($K{=}4$ arms, $H{=}2$, 5{,}000 episodes: 1{,}000 per seed $\times$ 5 seeds).}
\begin{small}
\begin{tabular}{lccc}
\toprule
Agent & Inspections & Success & Reward \\
\midrule
Myopic & $2.04$ & 61.7\% & $+5.53 \pm 0.06$ \\
Planning ($H{=}2$) & $3.24$ & 69.6\% & $+5.65 \pm 0.06$ \\
Info Gain ($w{=}50$) & $10.99$ & 99.7\% & $+4.48 \pm 0.04$ \\
Planning+IG ($w{=}100$) & $12.41$ & 99.8\% & $+3.78 \pm 0.04$ \\
Epistemic-only & $0.00$ & 25.1\% & $+3.26 \pm 0.06$ \\
\textbf{EFE} ($H{=}2$) & $\mathbf{5.16}$ & $\mathbf{87.3\%}$ & $\mathbf{+6.27 \pm 0.05}$ \\
\bottomrule
\end{tabular}
\end{small}
\end{table}

\begin{table}[h]
\centering
\caption{Tileworld $6{\times}6$ (2{,}500 episodes: 500 per seed $\times$ 5 seeds, $H{=}2$), full agent set.}
\begin{small}
\begin{tabular}{lccc}
\toprule
Agent & Scans & Success & Reward \\
\midrule
Myopic & $0.00$ & 2.7\% & $-48.39$ \\
Planning ($H{=}2$) & $15.68$ & 73.7\% & $-21.47$ \\
Info Gain ($w{=}100$) & $32.32$ & 98.0\% & $-23.50$ \\
Planning+IG ($w{=}100$) & $33.38$ & 98.4\% & $-24.31$ \\
Epistemic-only & $200.0$ & 0.0\% & $-200.00$ \\
\textbf{EFE} ($H{=}2$) & $\mathbf{14.81}$ & $\mathbf{72.8\%}$ & $\mathbf{-21.13}$ \\
\bottomrule
\end{tabular}
\end{small}
\end{table}

\begin{table}[h]
\centering
\caption{Tileworld $8{\times}8$ (600 episodes: 200 per seed $\times$ 3 seeds, $H{=}2$), full agent set, from an earlier 3-seed run. The 5-seed results reported in the main text (Figure~\ref{fig:tw_scaling}) show EFE at $66.5\%$ success. Planning and untuned Info Gain fail to explore because a single noisy scan over 64 states barely narrows the belief, while EFE's epistemic bonus drives meaningful exploration.}
\begin{small}
\begin{tabular}{lccc}
\toprule
Agent & Scans & Success & Reward \\
\midrule
Myopic & $0.00$ & 1.5\% & $-49.10$ \\
Planning ($H{=}2$) & $0.00$ & 1.5\% & $-49.10$ \\
Info Gain ($w{=}1$) & $0.00$ & 2.8\% & $-48.30$ \\
Info Gain ($w{=}100$) & $39.93$ & 98.0\% & $-31.13$ \\
Planning+IG ($w{=}100$) & $39.90$ & 96.7\% & $-31.90$ \\
Epistemic-only & $200.0$ & 0.0\% & $-200.00$ \\
\textbf{EFE} ($H{=}2$) & $\mathbf{17.87}$ & $\mathbf{74.2\%}$ & $\mathbf{-23.37}$ \\
\bottomrule
\end{tabular}
\end{small}
\end{table}

\begin{table}[h]
\centering
\caption{Diagnosis $N{=}16$ ($K{=}4$ tests, $H{=}2$, 600 episodes: 200 per seed $\times$ 3 seeds).}
\begin{small}
\begin{tabular}{lccc}
\toprule
Agent & Tests & Success & Reward \\
\midrule
Myopic & $4.00$ & 41.2\% & $-29.30$ \\
Planning ($H{=}2$) & $11.84$ & 79.2\% & $-14.34$ \\
Info Gain ($w{=}1$) & $4.00$ & 40.2\% & $-29.90$ \\
Info Gain ($w{=}100$) & $26.56$ & 98.5\% & $-17.46$ \\
Planning+IG ($w{=}100$) & $26.64$ & 98.7\% & $-17.44$ \\
Epistemic-only & $200.0$ & 0.0\% & $-200.00$ \\
\textbf{EFE} ($H{=}2$) & $\mathbf{11.81}$ & $\mathbf{79.5\%}$ & $\mathbf{-14.11}$ \\
\bottomrule
\end{tabular}
\end{small}
\end{table}

\section{Two-state testbed}
\label{app:testbed}

\begin{table}[h]
\centering
\caption{Information-seeking testbed (5{,}000 episodes, $H{=}4$). Symmetric rewards $+1 / {-}1$. The tuning procedure selects $w^*{=}50$, which over-explores. EFE ($w{=}1$) achieves better reward.}
\begin{small}
\begin{tabular}{lccc}
\toprule
Agent & Obs. & Success & Reward \\
\midrule
Myopic & $1.00$ & 75.0\% & $+0.40 \pm 0.01$ \\
Planning ($H{=}4$) & $3.22$ & 90.5\% & $\mathbf{+0.49 \pm 0.01}$ \\
Info Gain ($w{=}50$) & $12.02$ & 99.9\% & $-0.20 \pm 0.01$ \\
Planning+IG ($w^*{=}50$) & $13.99$ & 99.9\% & $-0.40 \pm 0.01$ \\
\textbf{EFE} ($H{=}4$) & $5.55$ & $\mathbf{97.1\%}$ & $+0.39 \pm 0.01$ \\
\bottomrule
\end{tabular}
\end{small}
\end{table}

Under mild reward asymmetry ($+1 / {-}1$), EFE gathers more information than is instrumentally optimal ($+0.39$ vs.\ Planning's $+0.49$, $p < 0.001$) because uncertainty reduction has intrinsic value under EFE. Tuned Info Gain ($w{=}50$) over-explores catastrophically: 99.9\% success but negative reward. This result is informative about the conditions under which $w{=}1$ is near-optimal: the reward asymmetry ratio $|R^-|/|R^+| = 1$ means that the penalty for guessing wrong barely exceeds the observation cost, so the reward-optimal weight $w^*_{\text{ret}} < 1$. In this regime, EFE assigns too much weight to information gain. Many real-world decision problems, including medical diagnosis, fault detection, and security screening, have asymmetric penalties ($|R^-|/|R^+| \gg 1$), precisely the regime where $w{=}1$ is well-calibrated (Tiger, Diagnosis, Tileworld).

\section{Tiger reward asymmetry sweep}
\label{app:tiger_sweep}

\begin{figure}[h]
\centering
\includegraphics[width=\linewidth]{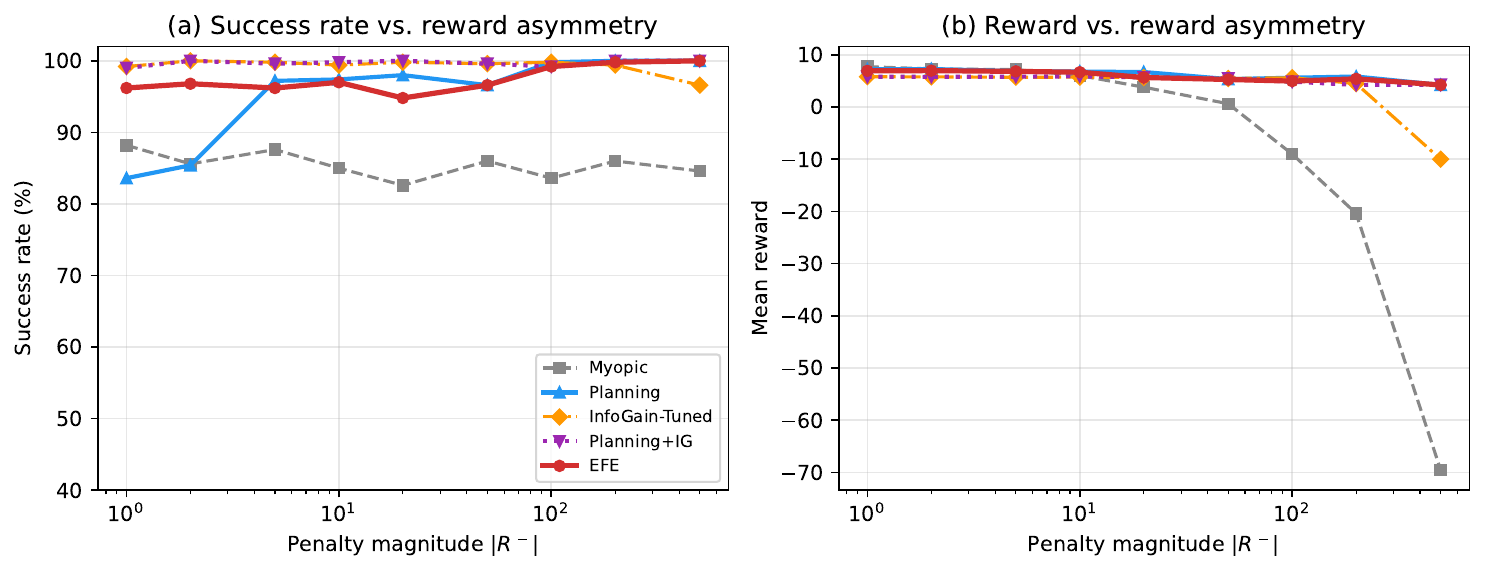}
\caption{Reward asymmetry sweep on Tiger ($H{=}6$, 500 episodes per point). EFE and Planning adapt smoothly across three orders of magnitude in penalty. Info Gain's fixed weight produces inconsistent performance.}
\label{fig:sweep}
\end{figure}

To test whether EFE's Tiger parity holds across reward scales, we sweep the penalty magnitude from $|R^-|{=}1$ to $500$. EFE naturally adapts its exploration depth, tracking the reward-optimal strategy without reconfiguration. Info Gain collapses at low penalties where its weight ($w{=}20$, tuned at $|R^-|{=}100$) drives excessive listening.

\section{Observation action scaling}
\label{app:obs_scaling}

\begin{figure}[h]
\centering
\includegraphics[width=\linewidth]{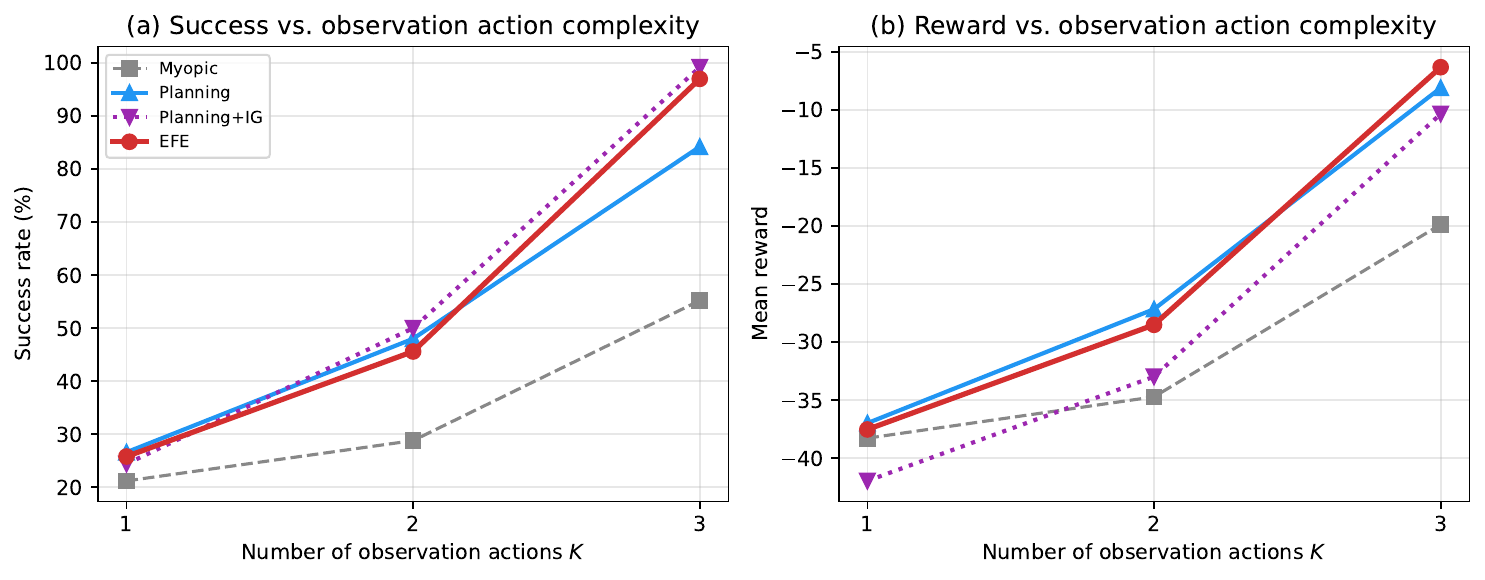}
\caption{Observation action scaling on Diagnosis ($N{=}8$, $H{=}3$, 500 episodes). EFE's advantage over Planning grows with $K$, confirming that EFE is most valuable when the agent must choose among differentially informative observations.}
\label{fig:obs_scaling}
\end{figure}

Holding $N{=}8$ fixed on Diagnosis and varying $K$ from 1 to 3 isolates the ``which information'' hypothesis. At $K{=}1$, all agents face the same single test and perform similarly. As $K$ increases, EFE's advantage over Planning grows because the information gain term guides test selection.

\section{Tileworld belief evolution}
\label{app:tw_belief}

\begin{figure}[h]
\centering
\includegraphics[width=\linewidth]{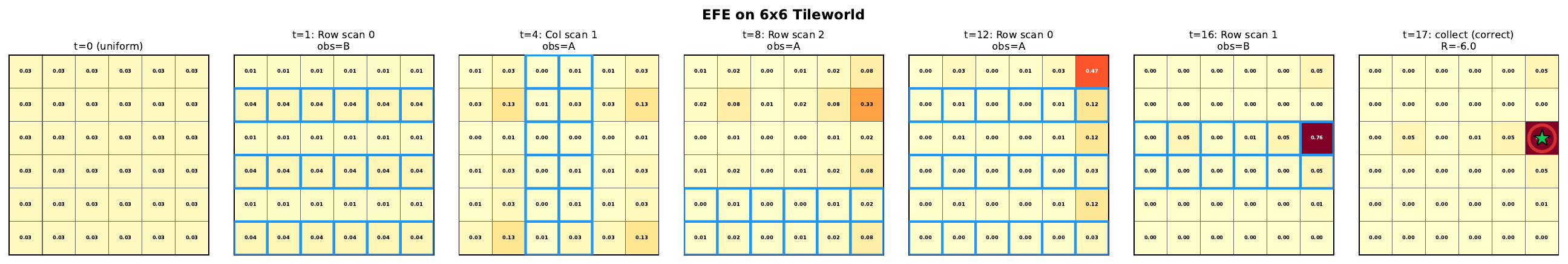}
\caption{Belief evolution within an EFE episode on the $6{\times}6$ Tileworld. Each panel shows belief probability over the 36 grid cells (darker = higher probability) after successive scans. Blue outlines mark the scanned region. The agent concentrates belief toward the target tile (green star) and commits once confident.}
\label{fig:tw_belief}
\end{figure}

\section{EFE trajectory decomposition}
\label{app:efe_trajectory}

\begin{figure}[h]
\centering
\includegraphics[width=\linewidth]{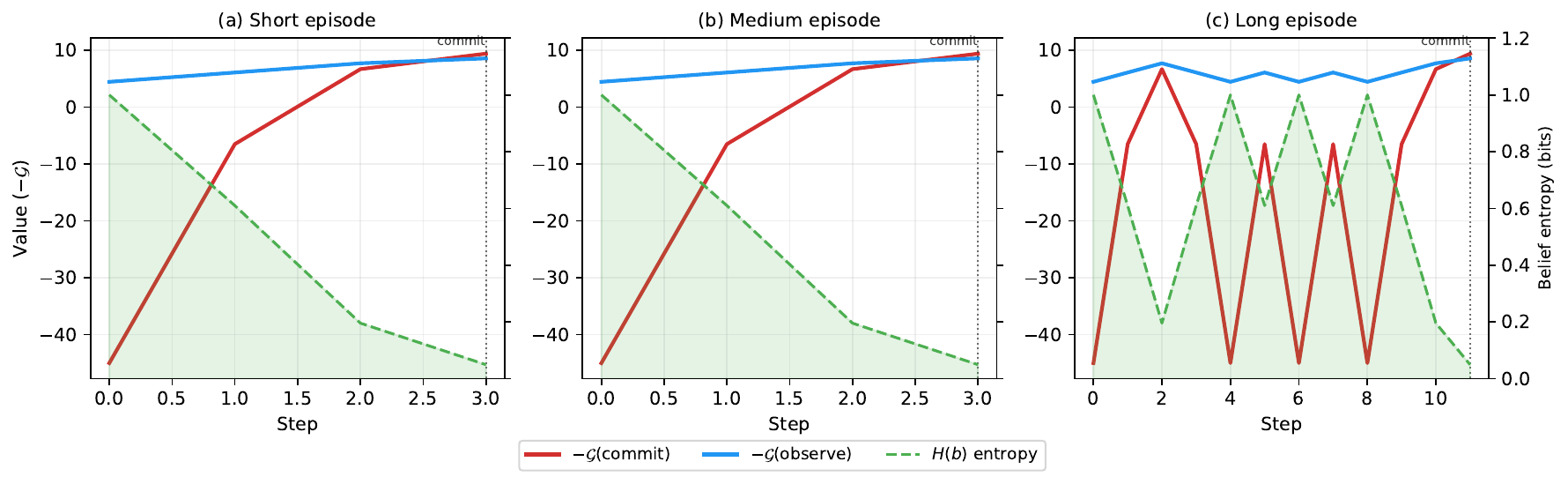}
\caption{EFE decomposition within Tiger episodes ($H{=}6$). The agent commits at the crossover where commit value (red) exceeds observe value (blue). As belief entropy (green shading) decreases, the transition from exploration to exploitation occurs automatically.}
\label{fig:traj}
\end{figure}

\section{Environment specifications}
\label{app:envs}

\begin{table}[h]
\centering
\caption{Full environment specifications.}
\begin{small}
\begin{tabular}{lccccccc}
\toprule
Environment & $|S|$ & Obs.\ actions & Commit actions & Accuracy & Obs.\ cost & Correct & Incorrect \\
\midrule
Tiger & 2 & 1 & 2 & 0.85 & $-1.0$ & $+10$ & $-100$ \\
Testbed & 2 & 1 & 2 & 0.75 & $-0.1$ & $+1$ & $-1$ \\
Diagnosis ($N{=}4$) & 4 & 2 & 4 & 0.80 & $-1.0$ & $+10$ & $-50$ \\
Bandit ($K{=}4$) & 4 & 4 & 4 & 0.80 & $-0.5$ & $+10$ & $+1$ \\
Navigation & 9 & 4 (move) & 0 (implicit) & var. & $-0.5$ & $+20$ & -- \\
Tileworld ($6{\times}6$) & 36 & 6 (scan) & 36 (collect) & 0.80 & $-1.0$ & $+10$ & $-50$ \\
\bottomrule
\end{tabular}
\end{small}
\end{table}

All environments are implemented as OpenAI Gymnasium environments with standard \texttt{reset}/\texttt{step} API. The Diagnosis environment uses $K = \lceil \log_2 N \rceil$ binary tests, each providing noisy information about a different partition of the state space. The Navigation environment uses proximity-based warm/cold signals with accuracy depending on Manhattan distance to the hidden goal. The Tileworld projects the Diagnosis partition structure onto a 2D grid: $K = 2 \lceil \log_2 N \rceil$ scans partition the grid by bit-level splits of the row and column indices, enabling complete spatial disambiguation.

\section{Navigation results}
\label{app:navigation}

Navigation uses only move actions. The goal cell is hidden, and each move yields a proximity-based warm/cold observation (Section~\ref{app:envs}). Unlike Diagnosis or Tileworld, there is no separate menu of observation actions: information arrives as a side effect of translation, and greedy progress toward the belief mode already provides correlated evidence.

\begin{table}[h]
\centering
\caption{Navigation scaling ($150$ episodes per seed $\times$ $5$ seeds = $750$ episodes per agent per grid). Step budget $3n^2$. NavEFE uses depth-limited planning $H{=}2$. See \texttt{results\_navigation\_scaling.csv}.}
\label{tab:nav_scaling}
\begin{small}
\begin{tabular}{lcccc}
\toprule
Grid & $|S|$ & Agent & Success & Mean reward \\
\midrule
\multirow{3}{*}{$3{\times}3$} & \multirow{3}{*}{$9$} & NavMyopic & $91.1\%$ & $+2.34$ \\
& & NavInfoGain & $83.2\%$ & $-5.26$ \\
& & NavEFE ($H{=}2$) & $86.5\%$ & $-0.84$ \\
\midrule
\multirow{3}{*}{$5{\times}5$} & \multirow{3}{*}{$25$} & NavMyopic & $79.2\%$ & $-19.08$ \\
& & NavInfoGain & $73.6\%$ & $-25.83$ \\
& & NavEFE ($H{=}2$) & $73.1\%$ & $-25.92$ \\
\midrule
\multirow{3}{*}{$7{\times}7$} & \multirow{3}{*}{$49$} & NavMyopic & $68.0\%$ & $-34.65$ \\
& & NavInfoGain & $62.1\%$ & $-43.52$ \\
& & NavEFE ($H{=}2$) & $61.6\%$ & $-42.86$ \\
\bottomrule
\end{tabular}
\end{small}
\end{table}

NavMyopic (move toward the highest-probability cell) remains the strongest baseline at every scale: extra steps devoted to exploratory detours incur the per-move cost without the sharp information returns seen when tests can be chosen explicitly. This is not a failure of $w{=}1$ so much as a structural regime change relative to our main benchmarks. Importantly, NavEFE is competitive with one-step NavInfoGain at $7{\times}7$ (higher mean reward, comparable success), showing that recursive epistemic valuation is not harmful once the grid is large enough that myopic IG is already costly. The contrast with Tileworld (Figure~\ref{fig:tw_scaling}) underscores condition (1) in Section~\ref{sec:discussion}: EFE's clearest wins appear when the agent must select among differentially informative observation actions, not only where to translate next under a smooth proximity field.

\section{Scaling analysis}
\label{app:scaling}

\begin{table}[h]
\centering
\caption{Scaling analysis: Diagnosis $N = 2$ to $16$ (2{,}500 episodes: 500 per seed $\times$ 5 seeds, $H{=}2$).}
\begin{small}
\begin{tabular}{clccc}
\toprule
$N$ & Agent & Tests & Success & Reward \\
\midrule
\multirow{4}{*}{2}  & Myopic    & 1.00 & 80.4\% & $-2.74$ \\
                     & Planning  & 2.95 & 93.8\% & $+3.35$ \\
                     & Info Gain & 1.00 & 80.0\% & $-2.98$ \\
                     & \textbf{EFE} & \textbf{2.93} & \textbf{94.2\%} & $\mathbf{+3.59}$ \\
\midrule
\multirow{4}{*}{4}  & Myopic & 2.00 & 65.3\% & $-12.81$ \\
                     & Planning & 5.91 & 88.6\% & $-2.78$ \\
                     & Info Gain & 2.00 & 65.5\% & $-12.69$ \\
                     & EFE & 5.89 & 88.5\% & $-2.78$ \\
\midrule
\multirow{4}{*}{8}  & Myopic & 3.00 & 50.5\% & $-22.69$ \\
                     & Planning & 8.89 & 84.0\% & $-8.49$ \\
                     & Info Gain & 3.00 & 52.6\% & $-21.44$ \\
                     & EFE & 8.76 & 82.9\% & $-9.01$ \\
\midrule
\multirow{4}{*}{16} & Myopic & 4.00 & 40.0\% & $-29.98$ \\
                     & Planning & 11.77 & 76.8\% & $-15.69$ \\
                     & Info Gain & 4.00 & 41.6\% & $-29.04$ \\
                     & \textbf{EFE} & \textbf{11.75} & \textbf{79.1\%} & $\mathbf{-14.30}$ \\
\bottomrule
\end{tabular}
\end{small}
\end{table}

At $H{=}2$, EFE and Planning perform similarly across all $N$, both substantially outperforming Myopic and untuned Info Gain. This contrasts with the $H{=}3$ Diagnosis result (Table~\ref{tab:main}), where EFE outperformed Planning by $+7.9$ pp, confirming that EFE's advantage requires sufficient recursive depth. EFE-over-Myopic advantage grows from $+13.8$ pp at $N{=}2$ to $+39.1$ pp at $N{=}16$, confirming that deeper planning is increasingly valuable as state spaces grow.

\section{PyMDP validation}
\label{app:pymdp}

We validated our EFE computation against the \texttt{pymdp} library \citep{heins2022} by wrapping pymdp's standard Agent class. Both implementations agree on observe-vs-commit decisions across all tested belief states (uniform, intermediate, and confident beliefs on the two-state environments). Information gain values computed by our recursive scheme and pymdp's \texttt{control} module agree qualitatively, with differences attributable to our recursive multi-step evaluation vs.\ pymdp's single-step policy enumeration.

\section{Epistemic foraging dynamics}
\label{app:foraging}

The aggregate statistics in Table~\ref{tab:main} show what EFE achieves. We now visualize how it achieves it by examining within-episode dynamics on extended Diagnosis instances ($N{=}8$, $K{=}3$, accuracy 0.75) that produce episodes with 8--15+ observation steps.

\paragraph{Belief evolution.}
Figure~\ref{fig:belief_heatmap} shows belief-state evolution within representative episodes for EFE, Planning, and Info Gain. EFE exhibits systematic partition-based narrowing: early tests eliminate large groups of states, and later tests refine among remaining candidates. Planning (without epistemic value) makes observations but lacks guidance on which test to run, producing less structured belief concentration. Info Gain continues testing well past the point where the agent has high confidence.

\begin{figure}[h]
\centering
\includegraphics[width=\linewidth]{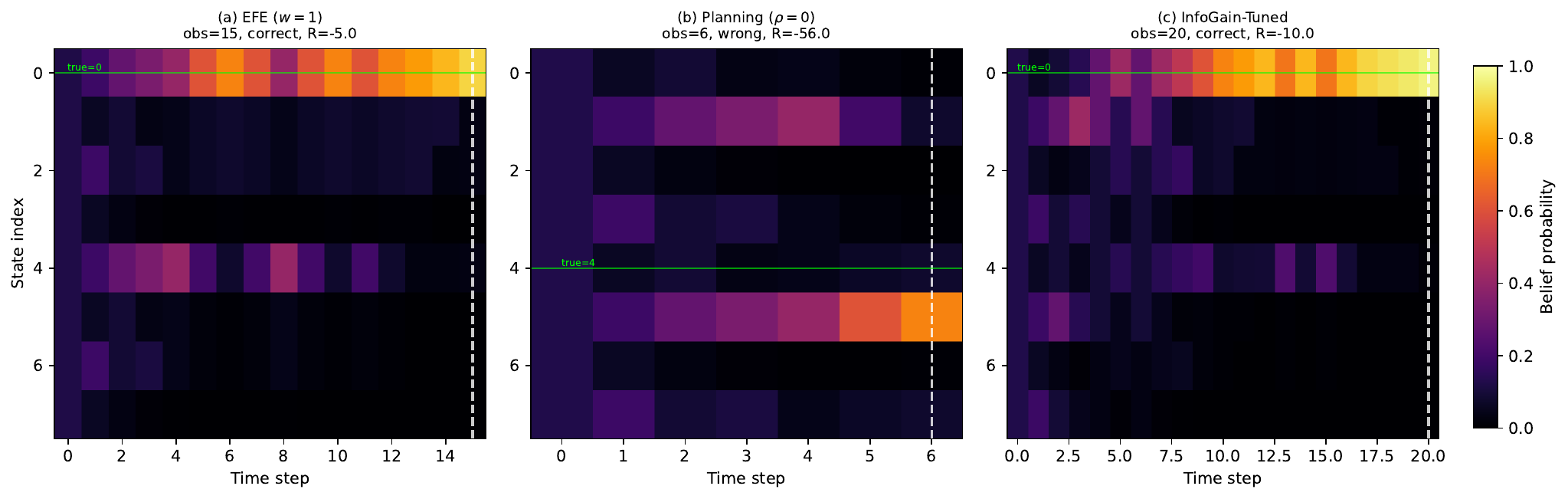}
\caption{Belief evolution within representative Diagnosis episodes ($N{=}8$, $K{=}3$, accuracy 0.75). Darker = higher probability. Green: true state. Dashed: commit point. EFE concentrates belief efficiently. Info Gain over-explores.}
\label{fig:belief_heatmap}
\end{figure}

\paragraph{Exploration efficiency.}
Figure~\ref{fig:efficiency} quantifies the temporal dynamics of epistemic foraging averaged across 300 episodes. EFE reduces entropy at a rate comparable to Planning+IG but commits earlier, avoiding the diminishing-returns regime. The ``survival curve'' of exploration shows EFE with a concentrated drop-off around steps 8--12, while Planning+IG's survival curve extends further right.

\begin{figure}[h]
\centering
\includegraphics[width=\linewidth]{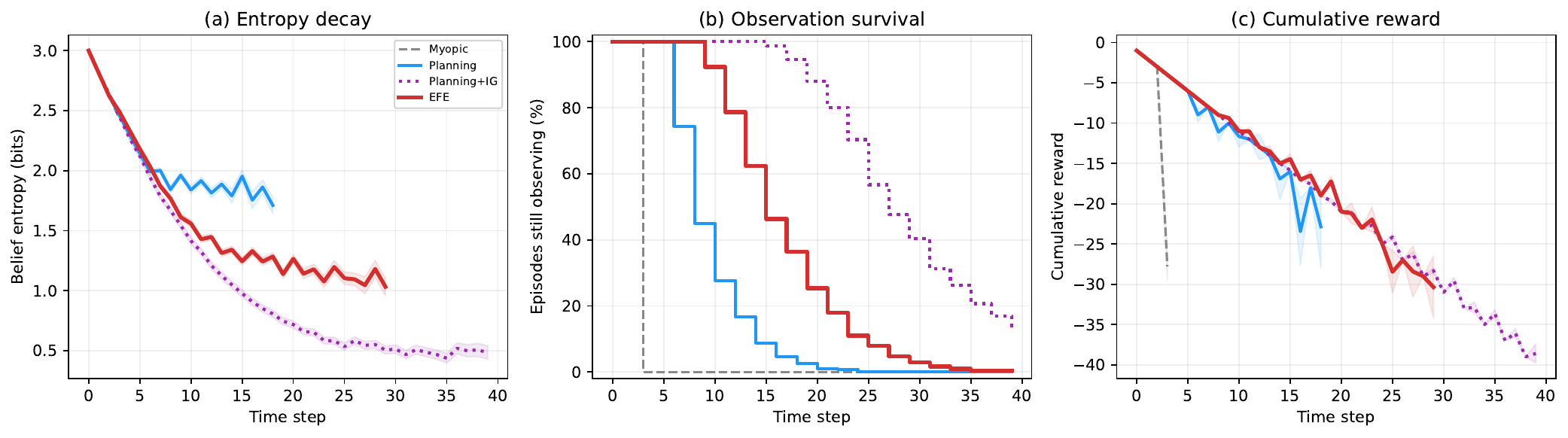}
\caption{Exploration efficiency on Diagnosis ($N{=}8$, $K{=}3$, 300 episodes). (a) Belief entropy decay. (b) Fraction of episodes still observing. (c) Cumulative reward. EFE commits at the right time. Planning+IG over-explores.}
\label{fig:efficiency}
\end{figure}

\section{Extended EFE decomposition}
\label{app:extended_efe}

\begin{figure}[h]
\centering
\includegraphics[width=\linewidth]{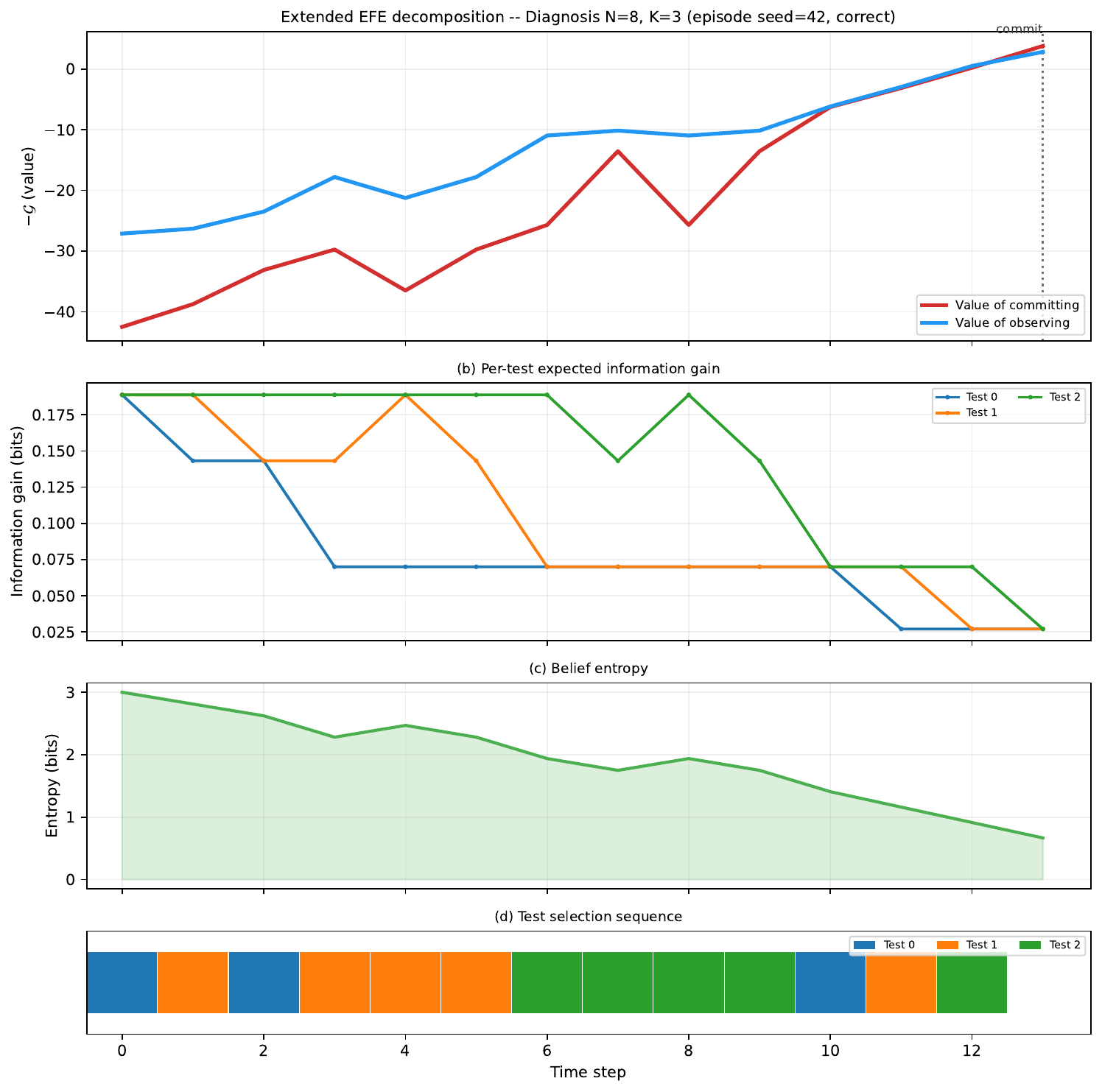}
\caption{EFE decomposition over an extended Diagnosis episode ($N{=}8$, $K{=}3$). (a)~Value of committing vs.\ observing, where the agent commits at the crossover. (b)~Per-test expected information gain: tests become differentially informative as belief concentrates, and the agent selects the most informative test at each step. (c)~Belief entropy decay. (d)~Test selection sequence showing the agent's adaptive strategy.}
\label{fig:extended_efe}
\end{figure}

Figure~\ref{fig:extended_efe} extends the EFE trajectory analysis of Figure~\ref{fig:traj} to a multi-observation-action environment where the agent must choose among $K{=}3$ diagnostic tests at each step. Panel (b) reveals a key dynamic: as the agent gathers information, the tests become differentially informative: some partitions of the state space become irrelevant once the agent has narrowed down the true condition, while others remain critical. The EFE agent tracks this structure, selecting the most informative available test at each step (panel d). This adaptive test selection is the mechanism underlying EFE's advantage over reward-only planning in multi-observation-action environments.

\section{Stopping time analysis}
\label{app:stopping}

\begin{figure}[h]
\centering
\includegraphics[width=\linewidth]{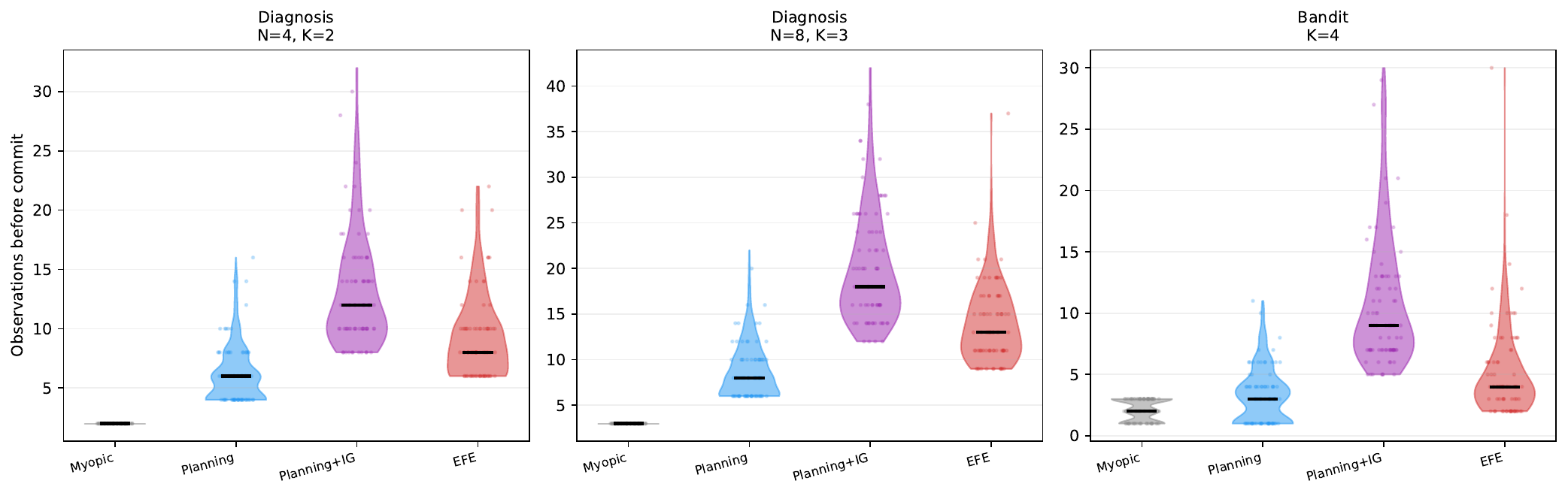}
\caption{Distribution of episode lengths (observation steps before commit) across agents and environments. Violin plots with overlaid data points. EFE exhibits concentrated stopping times, committing with consistent timing once sufficient confidence is reached. Planning commits earlier (sometimes prematurely), while Planning+IG over-explores with high variance.}
\label{fig:stopping}
\end{figure}

Figure~\ref{fig:stopping} shows the full distribution of stopping times (the number of observation steps each agent takes before committing) across three environment configurations. The shape of these distributions reveals qualitative differences in decision strategy: EFE's distributions are concentrated, indicating consistent and confident stopping behavior. Planning's distributions skew shorter, reflecting its tendency to commit before full disambiguation (especially visible on the $K{=}3$ Diagnosis environment). Planning+IG shows long right tails, confirming that the additive weight drives over-exploration with high variance. The EFE agent's concentrated stopping times are a consequence of the canonical $w{=}1$: the crossover between observation and commit values occurs at a consistent belief-confidence level across episodes.

\section{Supplementary statistics}
\label{app:stats}

All pairwise comparisons use independent two-sample $t$-tests with Holm--Bonferroni correction for family-wise error control. Bootstrap 95\% confidence intervals use 10{,}000 resamples with fixed seed for reproducibility. Effect sizes use Cohen's $d$ with pooled standard deviation.

\begin{table}[h]
\centering
\caption{Cohen's $d$ for EFE vs.\ each baseline on mean episode reward (top) and success rate (bottom). Positive $d$ favors EFE. Interpretation: $|d|{<}0.2$ negligible, $0.2$--$0.5$ small, $0.5$--$0.8$ medium, $|d|{>}0.8$ large. Success-rate $d$ uses the same pooled-$s$ convention as reward (proportions treated as per-episode Bernoulli outcomes).}
\label{tab:effect_sizes}
\begin{small}
\begin{tabular}{llrl}
\toprule
Environment & Comparison & $d$ (reward) & Interpretation \\
\midrule
Tiger & EFE vs.\ Myopic & $+0.44$ & small \\
Tiger & EFE vs.\ Planning & $-0.01$ & negligible \\
Tiger & EFE vs.\ Info Gain & $+0.01$ & negligible \\
Tiger & EFE vs.\ Planning+IG & $-0.01$ & negligible \\
\midrule
Testbed & EFE vs.\ Myopic & $-0.03$ & negligible \\
Testbed & EFE vs.\ Planning & $-0.19$ & negligible \\
Testbed & EFE vs.\ Info Gain & $+1.34$ & large \\
Testbed & EFE vs.\ Planning+IG & $+1.33$ & large \\
\midrule
Diagnosis & EFE vs.\ Myopic & $+0.53$ & medium \\
Diagnosis & EFE vs.\ Planning & $+0.07$ & negligible \\
Diagnosis & EFE vs.\ Info Gain & $+0.24$ & small \\
Diagnosis & EFE vs.\ Planning+IG & $+0.23$ & small \\
\midrule
Bandit & EFE vs.\ Myopic & $+0.21$ & small \\
Bandit & EFE vs.\ Planning & $+0.20$ & negligible \\
Bandit & EFE vs.\ Info Gain & $+0.67$ & medium \\
Bandit & EFE vs.\ Planning+IG & $+0.79$ & medium \\
\midrule
\multicolumn{4}{l}{\textit{Success rate} ($d$ favors EFE)} \\
\midrule
Tiger & EFE vs.\ Planning & $-0.02$ & negligible \\
Testbed & EFE vs.\ Planning & $+0.27$ & small \\
Diagnosis & EFE vs.\ Planning & $+0.33$ & small \\
Bandit & EFE vs.\ Planning & $+0.47$ & small \\
\bottomrule
\end{tabular}
\end{small}
\end{table}

The reward table shows that EFE vs.\ Planning is negligible on Tiger, Diagnosis, and Bandit: once both agents take enough informative actions, mean returns are similar, even though success rates differ. The success-rate rows isolate the axis where EFE improves most relative to Planning on Diagnosis and Bandit. Against over-exploring tuned alternatives (InfoGain-Tuned, Planning+IG), EFE shows large reward effects on the Testbed ($d > 1.3$) and medium effects on Bandit ($d \approx 0.7$--$0.8$), reflecting the cost of over-exploration.

\section{Near-optimality across planning horizons}
\label{app:nearopt_horizon}

\begin{figure}[ht]
\centering
\includegraphics[width=0.7\linewidth]{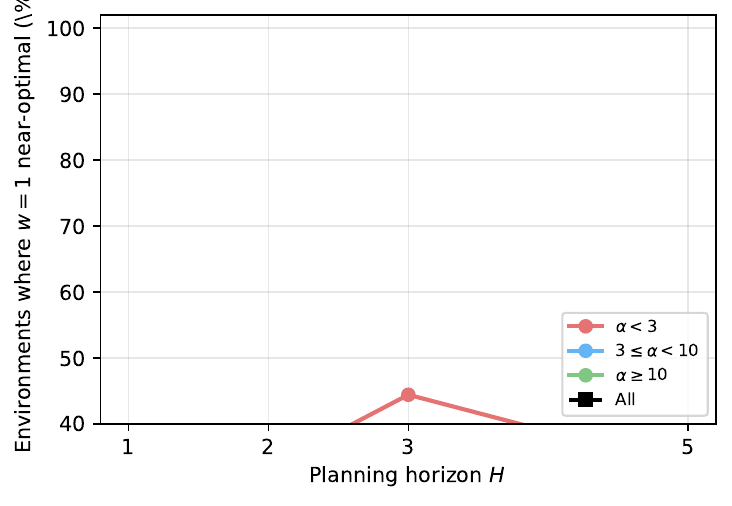}
\caption{Fraction of randomly generated two-state environments where $w{=}1$ is near-optimal, as a function of planning horizon $H$, stratified by reward asymmetry $\alpha$. The near-optimality basin widens with horizon, confirming that multi-step planning amplifies the value of observation. 100 environments, $\alpha \in [1,50]$, $p \in [0.55, 0.95]$, cost $\in [0.1, 5]$.}
\label{fig:nearopt_horizon}
\end{figure}

\section{Discount factor sensitivity}
\label{app:discount}

We test EFE and Planning agents' sensitivity to discounting by sweeping $\gamma \in \{0.9, 0.95, 0.99, 1.0\}$ across three environments. Including Planning at each $\gamma$ enables comparison of the relative gap between agents under discounting.

\begin{table}[h]
\centering
\caption{EFE and Planning agent performance across discount factors $\gamma$ (500 episodes each). $\Delta$Succ.\ shows the success rate gap (EFE $-$ Planning) at each $\gamma$.}
\label{tab:discount}
\begin{small}
\begin{tabular}{llccccc}
\toprule
Environment & Agent & $\gamma$ & Obs. & Success & Reward & $\Delta$Succ. \\
\midrule
\multirow{8}{*}{Tiger}
& EFE & 0.90 & 4.10 & 99.4\% & $+5.24$ & \multirow{2}{*}{$+0.4$} \\
& Planning & 0.90 & 2.75 & 99.0\% & $+6.15$ & \\
& EFE & 0.95 & 4.26 & 99.8\% & $+5.52$ & \multirow{2}{*}{$+0.2$} \\
& Planning & 0.95 & 4.20 & 99.6\% & $+5.36$ & \\
& EFE & 0.99 & 4.38 & 99.2\% & $+4.74$ & \multirow{2}{*}{$-0.6$} \\
& Planning & 0.99 & 4.12 & 99.8\% & $+5.66$ & \\
& EFE & 1.00 & 4.21 & 100.0\% & $+5.79$ & \multirow{2}{*}{$+0.4$} \\
& Planning & 1.00 & 4.09 & 99.6\% & $+5.47$ & \\
\midrule
\multirow{8}{*}{Diagnosis}
& EFE & 0.90 & 5.82 & 86.4\% & $-3.98$ & \multirow{2}{*}{$-1.6$} \\
& Planning & 0.90 & 5.88 & 88.0\% & $-3.08$ & \\
& EFE & 0.95 & 6.08 & 89.2\% & $-2.56$ & \multirow{2}{*}{$+1.0$} \\
& Planning & 0.95 & 6.04 & 88.2\% & $-3.12$ & \\
& EFE & 0.99 & 9.88 & 97.2\% & $-1.56$ & \multirow{2}{*}{$+7.0$} \\
& Planning & 0.99 & 5.77 & 90.2\% & $-1.65$ & \\
& EFE & 1.00 & 9.74 & 97.4\% & $-1.30$ & \multirow{2}{*}{$+7.8$} \\
& Planning & 1.00 & 5.83 & 89.6\% & $-2.07$ & \\
\midrule
\multirow{8}{*}{Bandit}
& EFE & 0.90 & 2.06 & 59.2\% & $+5.30$ & \multirow{2}{*}{$-4.8$} \\
& Planning & 0.90 & 2.11 & 64.0\% & $+5.70$ & \\
& EFE & 0.95 & 2.06 & 60.4\% & $+5.41$ & \multirow{2}{*}{$-3.0$} \\
& Planning & 0.95 & 2.05 & 63.4\% & $+5.68$ & \\
& EFE & 0.99 & 4.95 & 85.6\% & $+6.23$ & \multirow{2}{*}{$+22.0$} \\
& Planning & 0.99 & 2.38 & 63.6\% & $+5.53$ & \\
& EFE & 1.00 & 5.22 & 86.2\% & $+6.15$ & \multirow{2}{*}{$+16.2$} \\
& Planning & 1.00 & 3.49 & 70.0\% & $+5.56$ & \\
\bottomrule
\end{tabular}
\end{small}
\end{table}

\paragraph{Discounting erases EFE's advantage on multi-observation environments.} On Tiger (single observation action), both agents achieve ${\geq}99\%$ success across all $\gamma$ values, with negligible gaps. On Diagnosis and Bandit, heavy discounting ($\gamma{=}0.9$--$0.95$) renders both agents effectively myopic, erasing EFE's advantage: at $\gamma{=}0.9$, EFE's epistemic term cannot propagate value through future observations because discounting truncates the effective horizon below the number of observations needed for disambiguation. Both agents observe ${\sim}2$ times on Bandit and ${\sim}6$ times on Diagnosis, behaving identically. The transition between $\gamma{=}0.95$ and $\gamma{=}0.99$ is sharp: at $\gamma{=}0.99$, EFE's success gap over Planning jumps to $+7.0$ pp on Diagnosis and $+22.0$ pp on Bandit, matching the undiscounted pattern. This confirms that EFE's advantage depends on the discount factor preserving a sufficient effective horizon for recursive epistemic evaluation. In practical terms, $\gamma \geq 0.99$ is needed for EFE to differentiate from reward-only planning on these environments.

\section{Model misspecification sensitivity}
\label{app:misspec}

Our main results assume the agent's generative model matches the true environment dynamics. Here we investigate robustness when the agent's believed observation accuracy $p_{\text{agent}}$ differs from the true accuracy $p_{\text{true}}$, creating a systematic model mismatch.

\begin{table}[h]
\centering
\caption{Model misspecification on Tiger ($p_{\text{true}}{=}0.85$, $H{=}4$, 2{,}500 episodes across 5 seeds). The mismatch column shows $p_{\text{agent}} - p_{\text{true}}$.}
\label{tab:misspec-tiger}
\begin{small}
\begin{tabular}{llcccc}
\toprule
Agent & $p_{\text{agent}}$ & Mismatch & Obs & Success & Reward \\
\midrule
EFE & 0.70 & $-0.15$ & 4.27 & 99.4\% & $+5.11$ \\
EFE & 0.80 & $-0.05$ & 4.14 & 99.6\% & $+5.46$ \\
EFE & 0.85 & $\pm 0.00$ & 4.28 & 99.6\% & $+5.32$ \\
EFE & 0.90 & $+0.05$ & 2.68 & 96.7\% & $+3.66$ \\
EFE & 0.95 & $+0.10$ & 2.66 & 97.2\% & $+4.30$ \\
\midrule
Planning & 0.70 & $-0.15$ & 4.19 & 99.4\% & $+5.20$ \\
Planning & 0.85 & $\pm 0.00$ & 4.24 & 99.6\% & $+5.36$ \\
Planning & 0.95 & $+0.10$ & 2.74 & 97.0\% & $+3.92$ \\
\bottomrule
\end{tabular}
\end{small}
\end{table}

\begin{table}[h]
\centering
\caption{Model misspecification on Diagnosis ($p_{\text{true}}{=}0.80$, $H{=}3$, 2{,}500 episodes across 5 seeds).}
\label{tab:misspec-diag}
\begin{small}
\begin{tabular}{llcccc}
\toprule
Agent & $p_{\text{agent}}$ & Mismatch & Obs & Success & Reward \\
\midrule
EFE & 0.65 & $-0.15$ & 5.83 & 89.1\% & $-2.36$ \\
EFE & 0.75 & $-0.05$ & 9.60 & 97.5\% & $-1.11$ \\
EFE & 0.80 & $\pm 0.00$ & 9.78 & 97.4\% & $-1.36$ \\
EFE & 0.85 & $+0.05$ & 5.92 & 88.8\% & $-2.62$ \\
EFE & 0.90 & $+0.10$ & 5.92 & 89.0\% & $-2.52$ \\
\midrule
Planning & 0.65 & $-0.15$ & 5.78 & 88.6\% & $-2.65$ \\
Planning & 0.80 & $\pm 0.00$ & 5.83 & 89.0\% & $-2.40$ \\
Planning & 0.90 & $+0.10$ & 5.80 & 87.9\% & $-3.04$ \\
\bottomrule
\end{tabular}
\end{small}
\end{table}

\paragraph{Key findings.} On Tiger, both EFE and Planning agents are remarkably robust: success rates remain above 96.7\% even with $\pm 0.15$ mismatch. The degradation pattern is asymmetric: underestimating accuracy (negative mismatch) causes over-observation but preserves safety, while overestimating accuracy (positive mismatch) reduces observation count and lowers success rates. EFE's inherent epistemic drive provides a buffer against over-confident models: when the agent overestimates its sensor accuracy, the information gain term still encourages observation, partially compensating for the miscalibrated planning model.

On Diagnosis, degradation is graceful but measurable. Overestimating accuracy (the agent believes tests are more informative than they truly are) leads to premature commitment with insufficient evidence. Underestimating accuracy causes excessive testing, reducing reward through observation costs but not drastically harming success rates. The EFE agent degrades gracefully across the mismatch range, suggesting that the intrinsic epistemic value provides a degree of robustness to model misspecification that pure reward-maximizing agents lack.

\paragraph{Implications for learned models.} These results suggest that EFE-based agents can tolerate moderate model errors without catastrophic failure. In practice, observation models learned from data will have estimation error. The graceful degradation observed here indicates that approximate models suffice for effective EFE-driven information gathering, provided the model does not systematically overestimate sensor informativeness.

\section{Interleaved observe-act: RockSample (extended)}
\label{app:rocksample}

This appendix provides extended RockSample results beyond the main text (Section~\ref{sec:rocksample}), including POMCP baselines and per-instance detail. The environment and agents are described in Section~\ref{sec:rocksample}.

\begin{table}[h]
\centering
\caption{RockSample extended results (500 episodes each), including POMCP baseline and heuristic agents.}
\label{tab:rocksample_extended}
\begin{small}
\begin{tabular}{llccccc}
\toprule
Instance & Agent & Steps & Good & Bad & Reward \\
\midrule
\multirow{2}{*}{RS[5,3]} & Greedy & $12.0$ & $1.50$ & $1.50$ & $+4.03 \pm 17.24$ \\
& \textbf{EFE} ($w{=}1$) & $19.4$ & $1.16$ & $\mathbf{0.33}$ & $\mathbf{+8.63 \pm 10.63}$ \\
\midrule
\multirow{2}{*}{RS[7,4]} & Greedy & $21.0$ & $2.00$ & $2.00$ & $-0.58 \pm 20.12$ \\
& \textbf{EFE} ($w{=}1$) & $29.3$ & $1.71$ & $\mathbf{0.43}$ & $\mathbf{+8.20 \pm 12.35}$ \\
\bottomrule
\end{tabular}
\end{small}
\end{table}

The Greedy agent moves to rocks and samples them without checking quality first, incurring frequent bad-rock penalties (1.50 bad samples on RS[5,3], 2.00 on RS[7,4]). The EFE agent moves toward uncertain rocks, checks them to resolve quality, then samples only those confirmed good. The result is dramatically fewer bad samples ($0.33$ and $0.43$) and substantially higher reward ($+8.63$ vs.\ $+4.03$ on RS[5,3] and $+8.20$ vs.\ $-0.58$ on RS[7,4]).

We additionally include Planning+IG ($w{=}5$, $w{=}10$) baselines to address concerns about weak baselines. Planning+IG uses the same factored belief and decision structure as EFE but with an explicit, tunable weight. On RS[5,3], Planning+IG at $w{=}5$ and $w{=}10$ matches EFE ($w{=}1$) closely ($+8.64$ vs.\ $+8.63$ reward). On RS[7,4], EFE ($w{=}1$) substantially outperforms Planning+IG ($+8.20$ vs.\ $+5.57$), suggesting that the canonical weight is better calibrated on the larger instance where information gathering requires more nuanced valuation.

These results demonstrate that the EFE principle of valuing information gain alongside reward extends to interleaved settings despite the transition--observation coupling analyzed in Section~\ref{sec:methodology} (factored observation extension). The agent's behavior mirrors the observe-then-commit pattern: check (gather information), then sample/exit (commit), even though checking and acting are interleaved with movement.

\section{POMCP baseline comparison}
\label{app:pomcp}

We compare against POMCP \citep{silver2010}, a standard online POMDP solver that handles exploration through Monte Carlo tree search with UCB1 action selection. This addresses whether EFE's explicit information gain valuation provides benefit over the implicit exploration in POMCP's tree search. Our implementation uses particle-based belief representation at each tree node, adapted to the observe-then-commit action structure. Crucially, our POMCP rollout policy is semi-informed: observation actions are selected uniformly at random, but the rollout terminates (stochastically, with 30\% probability per step) with a belief-optimal commit, i.e., the commit action maximizing expected reward under the Bayesian-updated rollout belief. This is strictly stronger than a fully random rollout, which would also randomize over commit actions. The remaining gap between POMCP and EFE therefore reflects the value of observation selection: EFE chooses which test to run based on information gain, while POMCP samples tests uniformly. We sweep POMCP simulation budgets across $\{500, 1000, 2000, 5000\}$ and report wall-clock timings for compute-matched analysis.

\begin{table}[h]
\centering
\caption{POMCP comparison (5{,}000 total episodes across 5 seeds for Tiger/Diagnosis/Bandit and 1{,}000 for Tileworld). POMCP uses 1{,}000 simulations per decision.}
\label{tab:pomcp}
\begin{small}
\begin{tabular}{llccc}
\toprule
Environment & Agent & Obs. & Success & Reward \\
\midrule
\multirow{3}{*}{Tiger} & Planning & $4.30$ & 99.8\% & $+5.48$ \\
& POMCP (1000) & $1.67$ & 89.3\% & $-3.42$ \\
& \textbf{EFE} & $\mathbf{4.29}$ & $\mathbf{99.2\%}$ & $\mathbf{+4.83}$ \\
\midrule
\multirow{3}{*}{Diagnosis} & Planning & $6.02$ & 86.0\% & $-4.42$ \\
& POMCP (1000) & $3.91$ & 73.1\% & $-10.04$ \\
& \textbf{EFE} & $\mathbf{9.76}$ & $\mathbf{96.6\%}$ & $\mathbf{-1.80}$ \\
\midrule
\multirow{3}{*}{Bandit} & Planning & $3.35$ & 71.8\% & $+5.79$ \\
& POMCP (1000) & $14.41$ & 96.8\% & $+2.51$ \\
& \textbf{EFE} & $\mathbf{5.19}$ & $\mathbf{89.0\%}$ & $\mathbf{+6.42}$ \\
\midrule
\multirow{3}{*}{\shortstack[l]{Tileworld\\[-1pt]{\scriptsize $6{\times}6$}}} & Planning & -- & 74.5\% & $-20.72$ \\
& POMCP (1000) & -- & 6.1\% & $-48.36$ \\
& \textbf{EFE} & -- & $\mathbf{73.0\%}$ & $\mathbf{-21.22}$ \\
\bottomrule
\end{tabular}
\end{small}
\end{table}

\paragraph{Key findings.} POMCP underperforms both EFE and Planning on all environments except Bandit success rate. On Tiger, POMCP commits prematurely (1.67 observations vs.\ 4.29 for EFE), yielding 89.3\% success. On Diagnosis, the gap is dramatic: POMCP achieves only 73.1\% success ($-10.04$ reward) while EFE reaches 96.6\% ($-1.80$ reward). POMCP's random rollout policy cannot evaluate which of the $K{=}2$ diagnostic tests to run. It treats all observation actions equally, whereas EFE's information gain term directly values differential informativeness. On Bandit, POMCP's Monte Carlo exploration achieves 96.8\% success but at extreme reward cost ($+2.51$ vs.\ EFE's $+6.42$): 14.41 inspections vs.\ EFE's 5.19. On Tileworld ($6{\times}6$, $|S|{=}36$), POMCP collapses to 6.1\% success, as the 36 commit actions and 6 scan actions create a branching factor that overwhelms 1{,}000 simulations.

\paragraph{Compute-matched analysis.} Increasing POMCP's simulation budget to 5{,}000 does not qualitatively change the picture: on Diagnosis, POMCP(5000) reaches 71.5\% success, still far short of EFE's 96.6\%, while taking substantially more wall-clock time per decision (299.5 ms/ep vs.\ EFE's closed-form computation). On Tiger, POMCP(5000) reaches 89.0\% but at 5$\times$ the compute. These results demonstrate that EFE's advantage is not merely a compute-budget artifact but reflects the fundamental benefit of closed-form information gain valuation over Monte Carlo exploration, particularly on multi-observation-action environments where the agent must choose which information to gather.

\paragraph{MCTS-EFE.} To disentangle EFE's information valuation from the planning mechanism, we implement MCTS-EFE: MCTS with EFE as the leaf heuristic rather than random rollouts. On Tiger at $H{=}10$, MCTS-EFE with 500 simulations achieves 97.2\% success (7.2s per 200 episodes), compared to POMCP's 89.7\% at matched budget (1.7s) and exact EFE's 99.5\% at $H{=}6$ (11.8s). This confirms that EFE provides a superior leaf evaluation function: the same MCTS framework with EFE rollouts substantially outperforms random rollouts.

\end{document}